\documentclass{article}

    \usepackage[nonatbib,preprint]{neurips_2025}

\usepackage[utf8]{inputenc} % allow utf-8 input
\usepackage[T1]{fontenc}    % use 8-bit T1 fonts
\usepackage{hyperref}       % hyperlinks
\usepackage{url}            % simple URL typesetting
\usepackage{booktabs}       % professional-quality tables
\usepackage{amsfonts}       % blackboard math symbols
\usepackage{nicefrac}       % compact symbols for 1/2, etc.
\usepackage{microtype}      % microtypography
\usepackage[table]{xcolor}
\usepackage{xspace}
\usepackage{amsmath}
\usepackage{bbm}
\usepackage{subcaption}
\usepackage{tcolorbox}
\usepackage{listings}
\usepackage[ruled,vlined]{algorithm2e}
\usepackage{amsthm}
\usepackage{graphicx}
\theoremstyle{plain}
\newtheorem{theorem}{Theorem}[]  % Numbered by section
\newtheorem{lemma}[]{Lemma}      % Share numbering with theorem

\theoremstyle{definition}

\theoremstyle{remark}

\usepackage[
backend=biber,
style=alphabetic,
sorting=ynt
]{biblatex}
\title{Reviving Any-Subset Autoregressive Models with Principled Parallel Sampling and Speculative Decoding}

\newcommand{\asarm}{AS-ARMs\xspace}
\newcommand{\assd}{ASSD\xspace}

\author{%
Gabe Guo \\
Department of Computer Science \\
Stanford University \\
\texttt{gabeguo@stanford.edu}
\And
Stefano Ermon \\
Department of Computer Science \\
Stanford University \\
\texttt{ermon@cs.stanford.edu}
}

\begin{document}

\maketitle

\begin{abstract}
% Recently, there has been interest in language modeling strategies that can generate tokens in orders besides left-to-right. 
%Currently, discrete diffusion models are the dominant paradigm of arbitrary-order language generation, generating multiple tokens in parallel with just one network evaluation. 
%However, the more tokens that are generated in parallel, the less accurate their predicted distributions become, as they rely on a conditional independence assumption that only works with infinitesimally small timesteps. 
In arbitrary-order language models, it is an open question how to sample tokens in parallel from the \textit{correct joint distribution}.
With discrete diffusion models, the more tokens they generate in parallel, the less their predicted distributions adhere to the originally learned data distribution, as they rely on a conditional independence assumption that only works with infinitesimally small timesteps. 
We find that a different class of models, any-subset autoregressive models (\asarm), holds the solution. As implied by the name, \asarm can generate tokens in any order, and in parallel. Moreover, \asarm support parallelized joint probability density estimation, allowing them to correct their own parallel-generated token distributions, via our \textit{Any-Subset Speculative Decoding (\assd)} algorithm. \assd provably enables generation of tokens from the correct joint distribution, with the number of neural network calls upper bounded by the number of tokens predicted.
We empirically verify that \assd speeds up language generation, without sacrificing quality. Furthermore, we provide a mathematically justified scheme for training \asarm for generation, and show that \asarm achieve state-of-the-art performance among sub-200M parameter models on infilling benchmark tasks, and nearly match the performance of models 50X larger on code generation. 
Our theoretical and empirical results indicate that the once-forgotten \asarm are a promising direction of language modeling.
\end{abstract}

\section{Introduction}

Almost all the SoTA LLMs \cite{achiam2023gpt, touvron2023llama, liu2024deepseek, team2023gemini} are autoregressive, \textit{i.e.}, they only support left-to-right token generation. As a result, they suffer from two major problems: (1) they must generate tokens one-by-one, which limits their speed; (2) they cannot infill sequences in orders besides left-to-right (unless specialized training strategies are adopted \cite{bavarian2022efficient, fried2022incoder, roziere2023code}, but these are heuristic and not guaranteed to output the correct structure). 

Regarding models that inherently support infilling, there are discrete diffusion models \cite{austin2021structured, campbell2022continuous, lou2023discrete, sahoo2024simple} and any-order autoregressive models (AO-ARMs) \cite{yang2019xlnet, shih2022training}. Discrete diffusion models have the benefit of parallel sampling of multiple tokens at a time, which potentially speeds up generation, but at the cost of fidelity to the learned data distribution \cite{lou2023discrete}. On the other hand, fast sampling schemes have generally not been explored for AO-ARMs. %But, as we will develop later, they hold the potential for more accurate sampling, due to their ability to do direct density estimation \cite{yang2019xlnet, shih2022training}. \se{i would remove this last sentence, breaks the flow}

We explore the problem of fast parallel sampling from AO-ARMs, without any degradation in output quality. Particularly, we are inspired by speculative decoding techniques, which have accelerated generation in standard autoregressive models by using a draft model to quickly generate multiple tokens without any degradation. Then, the tokens are accepted or rejected based on their fidelity to the joint distribution as evaluated by the oracle model. Interestingly, the outputted distribution is provably the same as would have been obtained from sampling only from the expensive oracle model, while usually requiring fewer neural network forward passes \cite{chen2023accelerating, leviathan2023fast}. 
Crucial to speculative decoding are (1) density estimation from the oracle, and (2) a fast draft model. %If we can have these in AO-ARMs, we can achieve speculative decoding.
Indeed, AO-ARMs, by design, estimate joint probability density of sequences \cite{yang2019xlnet}.
Furthermore, AO-ARMs can even act as their own draft models; due to their architectural design and training objective, they can generate tokens in any order and in parallel \cite{yang2019xlnet, shih2022training}. 

As such, we propose \textit{Any-Subset Speculative Decoding (\assd)}, an algorithm that combines speculative decoding with AO-ARMs. \assd provably generates sequences from the true joint distribution learned by the oracle model. Furthermore, it is mathematically guaranteed to never increase the number of neural network evaluations, which means that it can only speed up generation without losing quality (as we empirically verify).
%We empirically verify that \assd indeed speeds up generation while maintaining the correct joint distribution, matching our theoretical analysis. 
We also provide a mathematically justified training scheme for AO-ARMs. Finally, we show that appropriately trained AO-ARMs achieve state-of-the-art performance among sub-200M parameter models (diffusion and autoregressive) on infilling benchmark tasks, and nearly match the performance of models 50X larger on code generation, while needing less than half the training tokens. \footnote{\textcolor{blue}{Models available at \href{https://huggingface.co/therealgabeguo/ASARM}{https://huggingface.co/therealgabeguo/ASARM}. Code available at \href{https://github.com/gabeguo/any-order-speculative-decoding}{https://github.com/gabeguo/any-order-speculative-decoding}.}}

\section{Background}

%\se{maybe start with standard AR models as warmup, then go to infilling and explain why it's hard}

\subsection{Autoregressive Models}

Given a text sequence $\mathbf{x} \sim \mathcal{D}$, where $\mathcal{D}$ is a data distribution, autoregressive (AR) models learn
\begin{equation}\label{eqn:ar_objective}
    p(\mathbf{x}) = \prod_{i = 0}^{|\mathbf{x}|-1}p(x_{i} | x_{0}, \ldots, x_{i-1}).
\end{equation}
%\se{product}
Crucially, the product rule factorization means that an AR model only needs to learn conditional distributions with support of size $O(S)$, where $S$ is the size of the vocabulary. This factorization also admits a straightforward generation strategy, where token $x_i$ is sampled from $p(\cdot | x_0, \ldots, x_{i-1})$, \textit{i.e.}, the previous tokens are used to produce the next \cite{achiam2023gpt}.
This also means that the prompt for an AR model must always be the prefix, \textit{i.e.}, $\mathbf{x}_{0:i}$ -- arbitrarily-located prompts are not supported.

\subsection{Infilling}

We consider infilling tasks, where the prompt is not necessarily the prefix, but can be arbitrarily-located. Examples of this are code generation, story completion, and scientific data imputation.
%We consider the task of conditional (\textit{i.e.}, from an arbitrary prompt) language generation. 
Mathematically, this can be formulated as sampling from a joint conditional probability distribution with discrete state space. That is, we wish to sample from $p(\mathbf{x}_{\sigma(\geq m)} | \mathbf{x}_{\sigma(<m)})$, where $\sigma(i)$ is the (zero-indexed) positional index of the $i$-th ordered item in the sequence of length $N$. In other words, $\sigma$ is a permutation of $\{0, 1, \ldots, N - 1\}$, and $i$ is the generation order. So, $\mathbf{x}_{\sigma(<m)}$ represents the prompt tokens, and $\mathbf{x}_{\sigma(\geq m)}$ represents the tokens whose distributions we want to predict. In general, we can have any $\sigma: \{0, \ldots, N-1\} \rightarrow \{0, \ldots, N-1\}$, so long as $\sigma$ is a bijection. 
As previously established, regular AR models cannot address this task in general, except when $\sigma(i) = i$.

%Typical autoregressive generation is actually a special case of this task. There, $\sigma(i) = i$, since generation and prompting are done left-to-right. Any-subset generation generalizes this, as the prompt does not necessarily have to be given on the left, but can be scattered at any position.

\subsection{Any-Order Autoregressive Models}\label{subsec:token_by_token}

%\se{this is a bit misleading, the key thing is that this can be done for any sigma, so it's really a collection of n! joints indexed by sigma}

Any-order autoregressive models (AO-ARMs) \cite{shih2022training, hoogeboom2021autoregressive, yang2019xlnet} can be seen as collections of $N!$ joint distributions indexed by $\sigma$, where $\sigma$ is the factorization order:
\begin{equation}\label{eqn:factorization}
    \begin{split}
        \text{log }p(\mathbf{x}_{\sigma(\geq m)} | \mathbf{x}_{\sigma(<m)}) %&= \text{log }\prod_{i=m}^{N}p(x_{\sigma(i)} | \mathbf{x}_{\sigma(<i)};\sigma) \\
        &= \sum_{i=m}^{N-1} \text{log } p(x_{\sigma(i)} | \mathbf{x}_{\sigma(<i)};\sigma).
    \end{split}
\end{equation}
That is, the distribution for each token is calculated one at a time, in order of increasing $i$, and used as conditioning for the next token to be predicted. Concretely, each evaluation of the network produces a probability distribution (for a single token) with support of size $O(S)$, where $S$ is the size of the vocabulary. The probability of each member of the support is explicitly calculated and stored in memory, incurring $O(S)$ memory cost per token.

%One perspective on AO-ARMs is that they are a collection of $N!$ joint distributions indexed by $\sigma$. 
If trained to optimality, all the joint distributions should be equal. However, except with infinite data and capacity, given different ordering functions $\alpha$ and $\sigma$ \cite{shih2022training},
\begin{equation}\label{eqn:inconsistent}
    \sum_{i=0}^{N-1} \text{log } p(x_{\sigma(i)} | \mathbf{x}_{\sigma(<i)};\sigma) 
    \neq \sum_{i=0}^{N-1} \text{log } p(x_{\alpha(i)} | \mathbf{x}_{\alpha(<i)};\alpha).
\end{equation}
%Again, this is a generalization of autoregressive models like GPT \cite{achiam2023gpt} or LLaMA \cite{touvron2023llama}, where $\sigma(i) = i$.

\subsection{Any-Subset Autoregressive Models: Disambiguating Joint Conditional Calculation}\label{subsec:efficient_mask}

%\se{this motivation is a bit confusing, how about saying it's a smaller collection of joints?}
Any-subset autoregressive models (\asarm) are a subclass of AO-ARMs that reduce the number of joint distributions ($\sigma$) learned from $N!$ to $2^N$. This puts less training burden on the finite model capacity, while keeping strictly the same expressivity as it relates to conditional joint distributions of the form in Equation \ref{eqn:factorization}.
% A known problem with AO-ARMs is that, except with infinite data and capacity, %trained all the way to optimality (which is generally not possible for nonconvex functions with gradient descent), \se{with infinite data, infinite capacity}
% \begin{equation}\label{eqn:inconsistent}
%     \sum_{i=m}^{N} \text{log } p(x_{\sigma(i)} | \mathbf{x}_{\sigma(<i)};\sigma) 
%     \neq \sum_{i=m}^{N} \text{log } p(x_{\alpha(i)} | \mathbf{x}_{\alpha(<i)};\alpha),
% \end{equation}
% where $\alpha()$ and $\sigma()$ represent two possible ordering functions with $\forall i < m: \sigma(i) = \alpha(i)$. In other words, given the same prompt, different factorization orders of the masked tokens ($i \geq m$) may give inconsistent probability density estimates. %Without careful consideration, this poses a major problem for the correctness of Algorithm \ref{alg:any_order_speculative}, because it would then be unclear if we actually match the "true" joint conditional distribution (Theorem \ref{thm:correct_dist}) when following a given decoding order.
To achieve this, \asarm adopt the recursive binary lattice mask decomposition protocol from \cite{shih2022training}. The idea underlying this protocol is that we can split every generation task into two parts: the prompt (denoted by $\mathbf{x}_{\sigma(<m)}$) and the tokens that need to be generated (denoted by $\mathbf{x}_{\sigma(\geq m)}$). Mathematically, we want to estimate $p(\mathbf{x}_{\sigma(\geq m)} | \mathbf{x}_{\sigma(<m)})$, where $m$ is the number of tokens in the prompt (assuming 0-indexing).
Now, we make two observations. 

%\se{imagine the reader has never seen this. be adversarial for misinterpretation}
Firstly, since we never need to evaluate the density of the conditioning $\mathbf{x}_{\sigma(<m)}$, we have every token attend to every other token within $\mathbf{x}_{\sigma(<m)}$. 
Secondly, within $\mathbf{x}_{\sigma(\geq m)}$, there is no need to learn all the possible factorization paths, as long as we can get the joint conditional probability $p(\mathbf{x}_{\sigma(\geq m)} | \mathbf{x}_{\sigma(<m)})$. Taking inspiration from vanilla autoregressive models, we enforce
\begin{equation}\label{eqn:left_to_right_forcing}
    \forall \ i \geq m, j \geq m: \sigma(i) > \sigma(j) \Leftrightarrow i > j.
\end{equation}
%\se{there are some subtelties here}
That is, we simply process the masked tokens left to right. Thus, given the location of the prompt, we only have to learn one path to calculate the joint probability of the generation. This also solves the inconsistency problem in Equation \ref{eqn:inconsistent}. 
(As later shown, this is crucial to Algorithm \ref{alg:any_order_speculative}'s correctness.)%, because it ensures that given a prompt, there is only one true possibility for the joint conditional distribution in Equation \ref{eqn:factorization}.) %If there were multiple paths to calculating the joint conditional, it would not be straightforward to say what the actual target distribution is.

%We emphasize this is different from vanilla autoregressive left-to-right decoding. In vanilla autoregressive models (like GPT), token index $a - 1$ must be known in order to decode token $a$. In \asarm, we only need to know token index $b < a$ to decode token $a$, even if $b \neq a-1$. Furthermore, \asarm can also condition token $a$ on token index $c > a$, whlie vanilla autoregressive models cannot. \se{this paragraph is confusing}

This reduces the number of queries learned by the model from $N!$ (all possible permutations) to $2^N$ (all possible mask location selections, times one ordering per mask selection). This makes optimization easier, as noted by \cite{shih2022training} and verified in our ablations (Figure \ref{fig:fixed_vs_any_order}). 
In summary, \asarm are a subclass of AO-ARMs incorporating Equation \ref{eqn:left_to_right_forcing}'s disambiguated ordering strategy. The architecture design is typically the same, but the way we query the architecture is different.

%In summary, any-subset autoregressive models (\asarm) are specially designed AO-ARMs that estimate joint conditional density via left-to-right ordering of the \textit{masked} tokens (Equation \ref{eqn:left_to_right_forcing}), as first proposed by \cite{shih2022training}.

\section{Sampling Strategies for Joint Distributions}\label{sec:sampling_strat}

%\se{it might be cleaner to start with a background section where you introduce the relevant concept, then define the problem}

% Given arbitrary conditioning tokens $\mathbf{x}_{\sigma(<m)}$ where $\sigma(i)$ is the positional index of the $i$-th ranked item in the sequence (\textit{i.e.}, $\sigma$ is a permutation of $\{0, 1, \ldots, N - 1\}$), we want to sample from the distribution  $p(\mathbf{x}_{\sigma(\geq m)} | \mathbf{x}_{\sigma(<m)})$. Working with log-probabilities, we wish to efficiently sample from

We want to accurately and efficiently sample from the joint conditional distribution in Equation \ref{eqn:factorization}. We assume that we have access to \asarm that explicitly predict single-variable marginals, \textit{i.e.}, next-token prediction under some ordering.

\noindent\textbf{{One-Step Sampling from the Joint:}}
Sampling directly from this joint distribution in a single step is typically infeasible, because its support has size $O(S^{N - m})$, where $S$ is the number of tokens in the vocabulary. If we wanted to explicitly calculate the probability of every tuple in the support, the exponential space cost to just store the distribution would quickly exceed memory capacities.

% To do this in one step is typically intractable, since the computational time cost of the probability mass function is exponential (with respect to the number of tokens generated): $O(S^{N - m})$, where $S$ is the number of tokens in the vocabulary.
% Thus, when sampling, there is a speed-fidelity tradeoff. 
% \se{this is a bit inaccurate, it really depends on how you are modeling the joint. in an ar model, it typically is (for most permutations), but for other models, it is not}

%\textcolor{red}{may have off-by-one error in indexing, depending on whether we 0-index or exclude the upper bound, etc. but the idea is here.}

\noindent\textbf{{Sequential Sampling via Factorization:}}
We can sample the $\text{log } p(x_{\sigma(i)} | \mathbf{x}_{\sigma(<i)})$ one-by-one, using each generated token as the conditioning for the next one. Following Equation \ref{eqn:factorization}, we can get samples from the joint conditional distribution by doing this $N-m$ times, once for each $i \in [m, N)$. In (any-order, any-subset) autoregressive models \cite{achiam2023gpt, touvron2023llama, yang2019xlnet, shih2022training}, there is typically $O(S)$ time cost per-token, so the overall time cost would be $O(S * (N-m))$. 
%incurring a linear computational time cost of $O(S * (N - m))$. This will give samples from the true joint conditional distribution, due to Equation \ref{eqn:factorization}.
Indeed, this is the dominant generation strategy for autoregressive models like GPT \cite{achiam2023gpt}, where $\sigma(i) = i$. %It can also be used in XLNet, an any-order autoregressive model, although this is not commonly used for generation \cite{yang2019xlnet}.

%\se{this subsection is confusing, everything really depends on how you represent the joint, and what is sigma (e.g., things are efficient if we have an ar model and signma is left-to-right}

%\se{probably better to use paragraphs instead of subsections, they are a bit too short}

\noindent\textbf{{Parallel Sampling via Independence Assumption:}}
At another extreme, we can independently sample multiple single-variable marginals in parallel. That is, we predict $\text{log } p(x_{\sigma(i)} | \mathbf{x}_{\sigma(<m)})$ for all $i \in [m, N)$. Since $N - m$ generations are done in parallel, computational time cost is only $O(S)$, which is effectively constant with respect to the number of tokens. The limitation is that 
\begin{equation}\label{eqn:bad_parallel}
    \sum_{i \in [m, N)} \text{log } p(x_{\sigma(i)} | \mathbf{x}_{\sigma(<m)}) \neq \text{log }p(\mathbf{x}_{\sigma(\geq m)} | \mathbf{x}_{\sigma(<m)}),
\end{equation}
except in the unlikely scenario of true independence, \textit{i.e.}, $\text{log } p(x_{\sigma(i)} | \mathbf{x}_{\sigma(<m)}) = \text{log } p(x_{\sigma(i)} | \mathbf{x}_{\sigma(<i)})$.

%\se{again, this is only possible if the model gives you these single variable marginals}

This is analogous to how discrete diffusion models take large discretized timesteps in the reverse CTMC to predict tokens in parallel, even though the predictions at each position are actually generated with a conditional independence assumption that only holds for an infinitesimally small timestep (\textit{i.e.}, one-by-one generation) \cite{lou2023discrete, sahoo2024simple}.

%A similar strategy (albeit with conditioning on the time variable in a CTMC) is popular in discrete diffusion models, like SEDD \cite{lou2023discrete}, MDLM \cite{sahoo2024simple}, and SDTT \cite{deschenaux2024beyond}. For this kind of sampling to work, the model architecture must provide the predictions for multiple tokens at a time, rather than just predicting the next token. This typically means that the length of the sequence is predefined, and the logits (or some bound) are predicted for all tokens at every generation step.
%\se{this paragraph has the same issue as the old intro, not quite accurate}

\noindent\textbf{{Best of Both Worlds: Any-Subset Speculative Decoding:}} %\label{sec:any_order_speculative_decoding}
We seek a way to combine the runtime benefits of parallel sampling with the fidelity of sequential sampling. That is, can we achieve $O(S)$ time complexity for arbitrary-subset generation, while recovering true samples from $\text{log }p(\mathbf{x}_{\sigma(\geq m)} | \mathbf{x}_{\sigma(<m)})$?
The key insight here is that if we had an oracle model that could evaluate the joint density of a sequence with a singular function evaluation, we could leverage the quick speed of independent parallel generation and use these samples $\text{log } p(x_{\sigma(i)} | \mathbf{x}_{\sigma(<m)})$ as estimates for the true $\text{log } p(x_{\sigma(i)} | \mathbf{x}_{\sigma(<i)})$. Then, via some rejection sampling scheme which uses as an oracle the joint density evaluation of this newly generated sequence, we could keep only the samples that adhere to $\text{log }p(\mathbf{x}_{\sigma(\geq m)} | \mathbf{x}_{\sigma(<m)})$. In principle, this could take best-case $o(S)$ time (parallelized), 
%regardless of number of tokens generated, 
if the oracle accepts all the samples, and worst-case $O(S * (N - m))$ time.

Speculative decoding is one such method that guarantees fidelity to the true distribution. It is mathematically proven to generate samples from the target distribution, and the empirical results show a stark decrease in the number of function evaluations required \cite{leviathan2023fast, chen2023accelerating}. However, prior to our work, speculative decoding has only been shown to work for traditional GPT-style \cite{achiam2023gpt} autoregressive models.

\section{Architectural Design of \asarm}

Towards the goal of fast, principled parallel sampling in any order, we seek any-subset autoregressive models (\asarm) that can support our \textit{Any-Subset Speculative Decoding (\assd)} algorithm (which is fully described in Section \ref{sec:any_order_speculative_decoding_algo_and_proof}). To recap, the criteria are: (1) generates arbitrarily-ordered tokens in parallel; (2) evaluates joint density with only one forward pass. We now show how \asarm can be designed to fulfill these criteria. In this section, we will focus on AO-ARM architecture, specifically XLNet \cite{yang2019xlnet}, as \asarm are architecturally the same as AO-ARMs.

%\textcolor{red}{talk about the architectural properties to speculative multiple tokens in parallel. this is a property of weight-tied AO-ARMs that satisfies this property. talk a little bit about how xlnet works under the hood}

\subsection{Parallel Sampling via Arbitrary Positional Queries}\label{subsec:parallel_sampling_xlnet}

Firstly, the architectures of modern weight-tied AO-ARMs \cite{yang2019xlnet, shih2022training} support parallel sampling. That is, with one function evaluation, we can simultaneously predict in parallel (conditionally independent) distributions for all the masked tokens $\mathbf{x}_{\sigma(\geq m)}$, conditioned on the prompt $\mathbf{x}_{\sigma(< m)}$. This allows the network to act as a quick "draft" model.

To illustrate how this works, we consider the XLNet architecture \cite{yang2019xlnet}. It is designed such that we can pass in arbitrary positional queries $\sigma(\geq m)$ of not-yet-predicted tokens to condition on the visible prompt tokens $\mathbf{x}_{\sigma(< m)}$. As implied by the "any-order" moniker, there is no constraint on which positions we query: we could query the leftmost, rightmost, or even a randomly selected unfilled position. %(This stands in direct contrast to autoregressive GPT-style models, which require the query to be for the next token to the right.)
Furthermore, the positional queries all attend to the same prompt tokens, but the prompt tokens cannot attend to the positional queries. So, no matter how many positions we query in parallel, each query cannot change the representations of the prompt tokens, and therefore cannot change the outcomes of the other simultaneous queries. This enables parallel sampling. 
See Figure \ref{subfig:parallel_sample_attention}. %For more details, we refer the reader to the original XLNet paper \cite{yang2019xlnet}.

\subsection{Quick and Principled Density Estimation via Causal-Like Attention Masking}\label{subsec:ao_arm_density}

\begin{figure}
    \centering
    \begin{subfigure}[]{0.4\textwidth}
        \includegraphics[width=\linewidth]{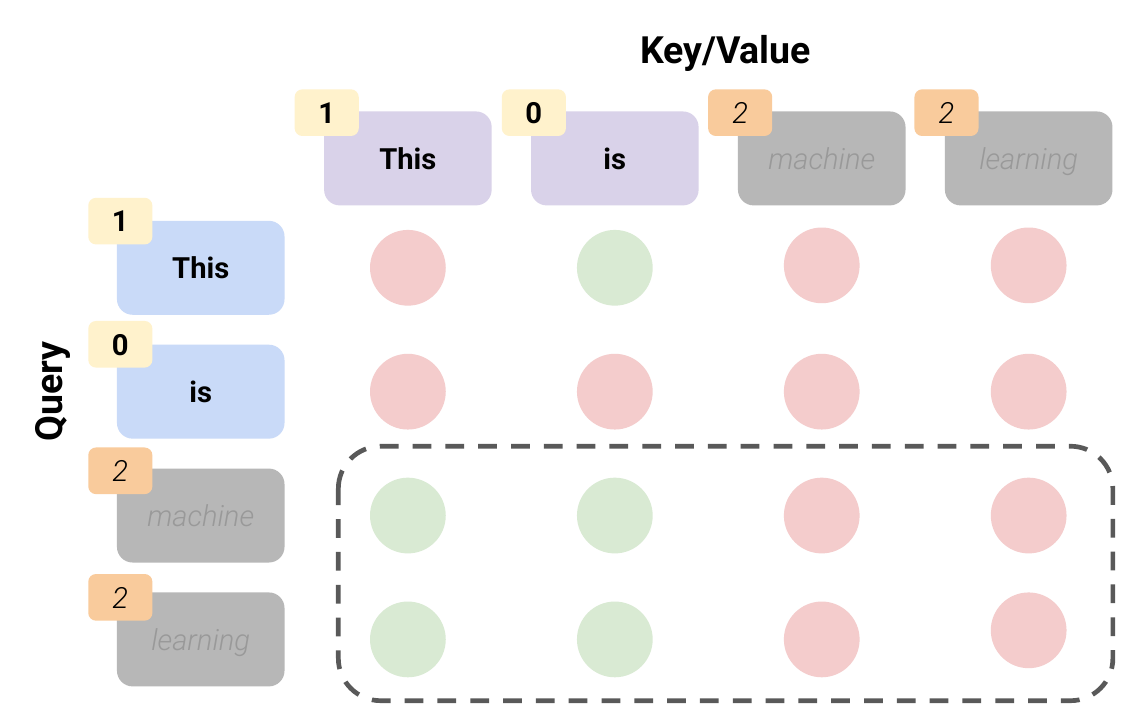}
        \caption{\textbf{Parallel Sampling Attention Mask}}\label{subfig:parallel_sample_attention}
    \end{subfigure}
    \begin{subfigure}[]{0.4\textwidth}
        \includegraphics[width=\linewidth]{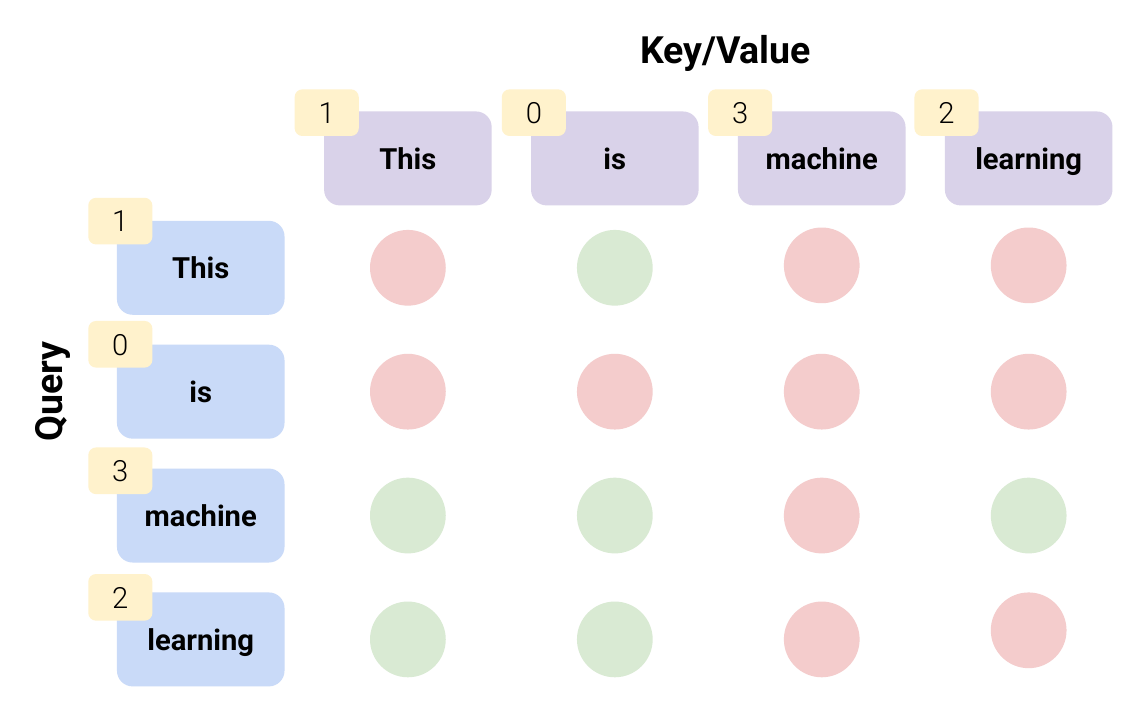}
        \caption{\textbf{Density Estimation Attention Mask}}\label{subfig:density_estimation_attention}
    \end{subfigure}
    \caption{\textbf{Attention Masks in \asarm:} Red means that attention is not allowed, while green means that attention is allowed from the query to the key/value. The numbers on each token represent the generation order, where lower-numbered tokens come first. Panel \ref{subfig:parallel_sample_attention} shows how, by attending to the same conditioning ("This", "is"), the tokens "machine" and "learning" can be generated in parallel. They do not attend to themselves nor each other. Panel \ref{subfig:density_estimation_attention} shows how we can conduct one-step density estimation on a sequence, with a permuted causal-like attention mask. The mask enforces the order "is", "This", "learning", "machine", where each token only attends to those decoded before it.
    %\se{could use some caption for what the numbers and colors mean}
    }
    \label{fig:attention_mask}
\end{figure}

Another crucial ingredient of speculative decoding is density estimation -- this allows the "oracle" model to correct the mistakes of the draft model. As such, discrete diffusion models trained with an ELBO \cite{lou2023discrete, sahoo2024simple, deschenaux2024beyond} are not readily adaptable to this scheme. 
However, AO-ARMs \cite{shih2022training, yang2019xlnet}, just like discrete diffusion models, can generate an arbitrary number of tokens in parallel in $O(S)$ time. Furthermore, they are designed to evaluate the true joint density of a sequence, something discrete diffusion models do not currently do. %Thus, they can support an appropriately modified speculative decoding algorithm.

Care must be taken, however, to pick an architecture that evaluates the joint density of a sequence in one function evaluation, \textit{i.e.}, $O(S)$ time. Recent architectures \cite{shih2022training, hoogeboom2021autoregressive} take $O(S * N)$ steps, as only logits of masked tokens are predicted at each function evaluation -- they are unable to predict the logits of visible tokens. 

%\se{i think we should explain why the difference}
%predict in the desired $O(S)$ time -- due to clever architectural design that separates content and positional information, the logits for all visible and masked tokens are calculated in a single function evaluation.

% \subsection{Attention Masking for Principled and Quick Density Estimation}

However, XLNet \cite{yang2019xlnet} can calculate the logits for all visible and masked tokens in a single function evaluation with $O(S)$ time. In particular, the attention masks for each token only allow attention (\textit{i.e.}, conditioning) to the preceding tokens in the ordering, such that we can construct a factorization as in Equation \ref{eqn:factorization} \cite{radford2019language}. %A token cannot attend to itself, as that would allow the model to  ``cheat'' when calculating the probability. 
Then, we process the tokens in parallel with this attention mask, thereby giving us the desired $O(S)$ time. Mathematically,

\begin{equation}\label{eqn:attn}
    A_{\sigma(i), \sigma(j)} = \begin{cases}
        0 & i \leq j \\
        1 & i > j
    \end{cases},
\end{equation}
where $A_{\sigma(i), \sigma(j)}$ is the masking matrix for the attention maps -- $0$ means that the query token at index $\sigma(i)$ is \textit{not} allowed to attend to the key/value token at index $\sigma(j)$; $1$ means that token $\sigma(i)$ \textit{can} attend to token $\sigma(j)$. Such an attention mask can yield faithful estimates of $\text{log } p(\mathbf{x}_{\sigma(\geq i)} | \mathbf{x}_{\sigma(<i)})$.
%The architectures for GPT \cite{achiam2023gpt}, LLaMA \cite{touvron2023llama}, and XLNet \cite{yang2019xlnet} support this manner of attention masking. However, out of these models, XLNet is the only architecture that supports an arbitrary $\sigma$, \textit{i.e.}, a sigma that supports orders besides sequential left-to-right generation. %(For GPT and LLaMA, it is necessary that $\sigma(i) = i$.) 
See Figure \ref{subfig:density_estimation_attention}, and Appendix \ref{sec:extra_causal_attention_masking} for further discussion.

%\se{i wonder if we should move this discussion to appendix, and leave a more succient summary in main paper}

\subsection{Two-for-One Model: Streamlined Drafting}

Ideally, we do not want to train a separate draft model, because it takes extra memory on the hardware and additional training cost. Furthermore, the computations from the separate draft model cannot necessarily be re-used for the oracle. 
%\se{we should probably cite some of the other ways to do similar things (self speculation) like layerskip}
If the draft model was the same as the target/oracle model, we would not need to incur an extra memory nor training cost, and could cache the computations from the draft model to accelerate the target calculations. %Indeed, works like LayerSkip \cite{elhoushi2024layerskip, zhang2023draft} use the same autoregressive model for drafting and verification, in a paradigm known as self-speculative decoding. 

Luckily, XLNet \cite{yang2019xlnet} allows us to predict multiple tokens independently in parallel; we can use this conditionally independent guess (as in Equation \ref{eqn:bad_parallel}) for the draft, like described in Section \ref{subsec:parallel_sampling_xlnet}. We can feed this draft back through the same model (with computations for the already-visible tokens cached) to calculate the oracle density estimates, as in Section \ref{subsec:ao_arm_density}. The draft tokens would then be accepted or rejected based on their fidelity to the oracle density, as described in the next section.

\section{How to Modify Speculative Decoding for \asarm?}\label{sec:any_order_speculative_decoding_algo_and_proof}

Now that we have established that there are suitable \asarm architectures, we present \textit{Any-Subset Speculative Decoding (\assd)}, and prove some of its properties. See Algorithm \ref{alg:any_order_speculative}. See Appendix \ref{sec:proofs} for proofs. In the resampling step, we use the notation $\left(f(x)\right)_{+} = \frac{\text{max}(0, f(x))}{\sum_{x}\text{max}(0, f(x))}$.
% \begin{equation}\label{eqn:mod_dist}
%     \left(f(x)\right)_{+} = \frac{\text{max}(0, f(x))}{\sum_{x}\text{max}(0, f(x))}.
% \end{equation}

%\se{add line numbers}

\begin{lemma}\label{lem:first_token_accept}
The first token speculated in each loop iteration will always be accepted. That is, Line \ref{line:accept}'s conditional always evaluates to true when $i = n$.
\end{lemma}

\begin{theorem}\label{lem:num_evals}
Algorithm \ref{alg:any_order_speculative} requires no more than $N - m$ total function evaluations of $p(\cdot | \cdot)$. That is, there will never be more calls to a neural network than the number of tokens returned on Line \ref{line:return}.
\end{theorem}

Based on Lemma \ref{lem:first_token_accept} and Theorem \ref{lem:num_evals} (see proof), we should always set $k > 2$ (where $k$ is the number of speculated tokens per call to the draft model). %, as it is guaranteed that at least two tokens get generated per loop iteration. %\se{k is undefined}

\begin{theorem}\label{thm:correct_dist}
Algorithm \ref{alg:any_order_speculative} produces samples from the true joint distribution $p(\mathbf{x}_{\sigma(\geq m)} | \mathbf{x}_{\sigma(<m)})$.
\end{theorem}

We also present a variant of \assd in Appendix \ref{subsec:speculative_decoding_experiment_details} with a context-derived n-gram as the draft model \cite{stewart2024n}. However, this variant does not fulfill Lemma \ref{lem:first_token_accept}.

\begin{algorithm}
\caption{Any-Subset Speculative Decoding}\label{alg:any_order_speculative}
\SetAlgoLined
\LinesNumbered
\KwIn{$k$: number parallel tokens; $N$: target sequence length; $\sigma$: mapping of decoding order to 0-based positional index; $p(\cdot | \cdot)$: any-order autoregressive model; $\mathbf{x}_{\sigma(<m)}$: prompt tokens}
\KwOut{$\mathbf{x}_{\sigma(\geq m)}$: the predicted tokens}

$n \leftarrow m$ // number of tokens we've already decoded \\
\While{n < N}{\label{line:main_loop}
    $t \leftarrow \text{min}(n+k, N)$ \\
    // speculate the next $t$ tokens \\
    \textbf{(Parallelized):} \For{$i \in [n:t)$}{\label{line:speculate_loop}
        $\tilde{x}_{\sigma(i)} \sim p(\cdot | \mathbf{x}_{\sigma(<n)})$ // sample from partially conditioned distribution \label{line:speculate_iter}\\
        $p_{\sigma(i)} \leftarrow p(\tilde{x}_{\sigma(i)} | \mathbf{x}_{\sigma(<n)})$ // get partially conditioned density\label{line:independent_density}
    }\label{line:end_speculate_loop}
    \If{n == T - 1}{\label{line:early_stop}
        $x_{\sigma(n)} \leftarrow \tilde{x}_{\sigma(n)}$ // accept the proposal \label{line:default_accept} \\
        \textbf{return} $\mathbf{x}_{\sigma(\geq m)}$
    }
    \textbf{(Parallelized):} \For{$i \in [n:t)$}{\label{line:oracle_loop}
        $q_{\sigma(i)} \leftarrow p(\tilde{x}_{\sigma(i)} | \mathbf{x}_{\sigma(<n)}, \mathbf{\tilde{x}}_{\sigma[n:i)})$ // get ground truth density\label{line:gt_density}
    }\label{line:end_oracle}
    // rejection sampling \\
    \For{$i \in [n: t)$}{\label{line:accept_reject_loop}
        $r \sim \mathcal{U}[0, 1]$ \\
        \eIf{$r < \text{min}(1, \frac{q_{\sigma(i)}}{p_{\sigma(i)}})$}{\label{line:accept}
            $x_{\sigma(i)} \leftarrow \tilde{x}_{\sigma(i)}$ // accept the proposal \label{line:accept_proposal_token}
        }{\label{line:reject}
            $x_{\sigma(i)} \sim \left(p(\cdot | \mathbf{x}_{\sigma(<n)}, \mathbf{\tilde{x}}_{\sigma[n:i)}) - p(\cdot | \mathbf{x}_{\sigma(<n)})\right)_{+}$ // resample \label{line:resample} \\
            \textbf{exit} from for loop
        }
    }\label{line:end_accept_reject}
    $n \leftarrow i + 1$ // update number of decoded tokens
}

\textbf{return} $\mathbf{x}_{\sigma(\geq m)}$ \label{line:return}

\end{algorithm}

\section{Training and Implementation of Any-Subset Autoregressive Model}

\subsection{Arbitrary Masking Architecture}

As expressed in Equation \ref{eqn:attn} and Section \ref{subsec:ao_arm_density}, we need a model that allows us to have "causal"-like attention masking in arbitrary orders, such that we can calculate the factorized joint probability in parallel. %While transformer decoders like GPT \cite{achiam2023gpt} and LLaMA \cite{touvron2023llama} have causal masking, they do not support arbitrary orderings. 
To our knowledge, XLNet \cite{yang2019xlnet} is the only architecture that fulfills this criteria. Thus, we use the 110M parameter case-sensitive version of XLNet from Huggingface \cite{wolf2019huggingface}.
Hyperparameters and datasets are in Appendix \ref{sec:additional_experimental_details}.

\subsection{Teacher-Forced Joint Loss}

\begin{figure}[h]
    \centering
    \includegraphics[width=0.65\linewidth]{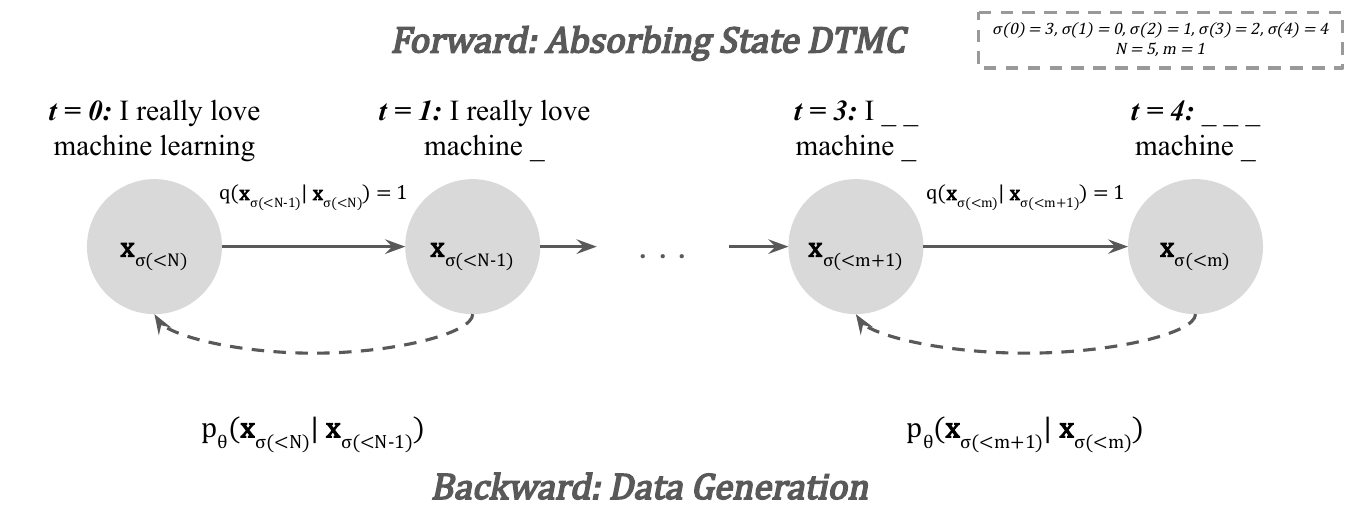}
    \caption{\textbf{Probabilistic Graphical Model:} Shows the discrete-time Markov chain for the forward noising process, and its time reversal (\textit{i.e.}, data generation). This justifies Equation \ref{eqn:loss}. %. By the factorization rule, this justifies the joint conditional probability in Equation \ref{eqn:loss} as the loss function.
    }
    \label{fig:probabilistic_graphical_model}
\end{figure}

The original XLNet \cite{yang2019xlnet} was only trained to predict $85$ masked tokens in a sequence length of $512$, corresponding to less than $20\%$ of the sequence, which is not ideal for generative modeling tasks \cite{yang2019xlnet}. We wish to have a model that can predict tokens almost from scratch, so we finetune the off-the-shelf XLNet model.
For our finetuning objective, we maximize the joint conditional probability in Equation \ref{eqn:factorization} with cross-entropy loss:
\begin{equation}\label{eqn:loss}
    \text{max}_\theta \ \mathbb{E}_{m \sim f(\cdot), \sigma \sim s(\cdot |m)}\left[\text{log }p_{\theta}(\mathbf{x}_{\sigma(\geq m)} | \mathbf{x}_{\sigma(<m)})\right],
\end{equation}
where $f(\cdot)$ is a distribution over integers from $[0, N)$ (\textit{i.e.}, sampling prompt length $m$), and $s(\cdot | m)$ is a distribution of permutations of integers from $[0, N)$ conditioned on prompt length $m$ (\textit{i.e.}, sampling token ordering $\sigma$), where $N$ is the sequence length.
The loss has three major components: (1) joint conditional distribution; (2) expectation over token orderings; (3) expectation over prompt lengths. 

\noindent\textbf{Joint Conditional Objective:}
To justify the joint conditional distribution $\text{log }p(\mathbf{x}_{\sigma(\geq m)} | \mathbf{x}_{\sigma(<m)})$, assume $m$ and $\sigma$ are fixed. We define a discrete time (absorbing state) Markov chain $\mathbf{x}, \mathbf{x}_{\sigma(<N-1)}, \mathbf{x}_{\sigma(<N-2)}, \ldots, \mathbf{x}_{\sigma(<m+1)}, \mathbf{x}_{\sigma(<m)}$, with time index $t \in \{0, 1, 2, \ldots, N-m-1, N-m\}$, as in Figure \ref{fig:probabilistic_graphical_model}. That is, $X_t = \mathbf{x}_{\sigma(<N-t)}$. To generate data, we follow the time reversal of this Markov chain. To obtain the time reversal, first consider reversing a singular time step, from $t$ to $t-1$. This corresponds to learning
\begin{equation} \label{eqn:reverse_step}
\begin{split}
    \text{log } p_{\theta}(X_{t-1} | X_{t}) = \text{log } p_{\theta}(\mathbf{x}_{\sigma(<N-(t-1))} | \mathbf{x}_{\sigma(<N-t)})
    %&= \text{log } p_{\theta}(\mathbf{x}_{\sigma(<N-t)}, \mathbf{x}_{\sigma(N-(t-1))} | \mathbf{x}_{\sigma(<N-t)}) \\
    = \text{log } p_{\theta}(\mathbf{x}_{\sigma(N-t+1)} | \mathbf{x}_{\sigma(<N-t)}).
\end{split}
\end{equation}
That is, we predict the next token's density.
To reverse the whole process, we sum for each timestep: 
\begin{equation} \label{eqn:reverse_joint}
\begin{split}
    \sum_{t=1}^{N-m}\text{log } p_{\theta}(\mathbf{x}_{\sigma(<N-t+1)} | \mathbf{x}_{\sigma(<N-t)}) &= \text{log } p_{\theta}(\mathbf{x}_{\sigma(<N)} | \mathbf{x}_{\sigma(<m)}) \\
    %&= \text{log } p_{\theta}(\mathbf{x}| \mathbf{x}_{\sigma(<m)}) \\
    = \text{log } p_{\theta}(\mathbf{x}_{\sigma(<m)}, \mathbf{x}_{\sigma(\geq m)} | \mathbf{x}_{\sigma(<m)}) %\\
    &= \text{log } p_{\theta}(\mathbf{x}_{\sigma(\geq m)} | \mathbf{x}_{\sigma(<m)}),
\end{split}
\end{equation}
which gives us Equation \ref{eqn:loss}'s joint conditional probability.
%This is simpler than the ELBOs for diffusion models \cite{sahoo2024simple, lou2023discrete, ho2020denoising}, because we fix the prompt length and (de)noising order. Conditioned on these two aspects, the forward process is deterministic.
This loss is different than the conditionally independent losses used in \cite{shih2022training} and discrete diffusion models \cite{sahoo2024simple, lou2023discrete}. Notably, their architectures, due to the lack of causal-like attention masking, could not support joint losses.
% We note that other works, such as \cite{shih2022training} and the discrete diffusion models \cite{sahoo2024simple, lou2023discrete} maximize the conditionally independent probabilities, as in Equation \ref{eqn:bad_parallel}. We did some preliminary experiments with this objective, and found that it generally led to less stable training and worse validation metrics, depending on the implementation. It is also important to note that due to the lack of causal-like attention masking in the previous works, their models were literally unable to train on a joint loss.

\noindent\textbf{Expectations over Token Orderings and Prompt Lengths:}
In the objective (Equation \ref{eqn:loss}), we do not make assumptions about the distribution of the prompt length $m \sim f(\cdot)$ nor the distribution of the prompt ordering $\sigma \sim s(\cdot | m)$ (so long as $\sigma$ follows the decomposition protocol laid out in Section \ref{subsec:efficient_mask}). In general, these distributions are task-dependent. For instance, in regular autoregressive tasks, $\sigma$ would be deterministic, \textit{i.e.}, the identity function. See Appendices \ref{subsec:mask_distribution} and \ref{sec:ablations} for our realizations of $m \sim f(\cdot)$ and $\sigma \sim s(\cdot | m)$.

\section{Experiments}
Our first experiment (Section \ref{subsec:empirical_verify}) empirically verifies that \assd with \asarm indeed preserves the output distribution while being faster, as predicted by our theoretical analysis.
Our other experiments (Section \ref{sec:infilling_benchmark}) show that on both natural language and coding benchmarks, \asarm, even with a fraction of the training resources, beat other model classes of comparable size, and are competitive with models orders of magnitude larger.
%\se{say why we want to do these experiments}
%\se{have a precise and attractive summary of results in abstract}

\subsection{Correctness and Speed of Any-Subset Speculative Decoding}\label{subsec:empirical_verify}

\begin{table}[]
    \centering
    \begin{tabular}{l r r r r r}
        \rowcolor{gray!35} \textbf{Sampler} & \textbf{Gen PPL} & \textbf{Entropy} & \textbf{Model NFE} & \textbf{Aux NFE} & \textbf{Time (s)} \\
        \toprule
        \textit{Sequential} & $107.9 \pm 1.6$ & $7.65 \pm 0.01$ & $486.0 \pm 0.0$ & $\mathbf{0.0 \pm 0.0}$ & $18.21 \pm 0.00$ \\
        \textit{\assd (N-Gram)} & $111.7 \pm 2.0$ & $7.64 \pm 0.01$ & $\mathbf{422.0 \pm 0.7}$ & $422.0 \pm 0.7$ & $16.80 \pm 0.03$ \\
        \textbf{\textit{\assd (Self)}} & $107.6 \pm 1.6$ & $7.64 \pm 0.01$ & $\underline{434.1} \pm 0.4$ & $\mathbf{0.0 \pm 0.0}$ & $\mathbf{16.50 \pm 0.02}$ \\
        % \midrule 
        % \textit{Difference} & $+1.91\%$ & $+0.14\%$ & $-10.87\%$ & $-8.65\%$ \\
        \bottomrule
    \end{tabular}
    \caption{\textbf{Comparison of Speculative and Sequential Decoding:} Left-to-right, entries show mean and standard error of generative perplexity (judge: GPT-2 Large), Shannon entropy, number of AS-ARM function evaluations, number of auxiliary draft model calls, and wall clock time. %Metrics are calculated over $640$ decoded WikiText sequences of length $512$, where $95\%$ is randomly masked out. 
    \textit{\assd (Self)} is from Algorithm \ref{alg:any_order_speculative}. We also modify Algorithm \ref{alg:any_order_speculative} with context-derived \textit{N-Grams} \cite{stewart2024n} as a draft model (see Appendix \ref{subsec:speculative_decoding_experiment_details}). We set $k = 5$ for \assd.}
    \label{tab:speculative_decoding}
\end{table}

We prompt the model with $640$ masked sequences from the WikiText test dataset \cite{wikitext_merity2016pointer}. As in training, sequences are packed together into chunks of $512$ tokens. We randomly mask out $95\%$ of each sequence, leaving $5\%$ of tokens (uniformly scattered throughout the sequence) as the prompt. We evaluate our finetuned model.
See Table \ref{tab:speculative_decoding} for results. Appendix \ref{sec:sample_outputs} shows sample outputs.

\noindent\textbf{Empirical Correctness of Distribution:}
As expected from Theorem \ref{thm:correct_dist}, the output distribution from speculative decoding is statistically the same as the output distribution from sequential decoding, measured by generative perplexity and entropy. 

\noindent\textbf{Speed-Up:}
Finally, both variants of \assd (parallel sampling from self as draft, context n-gram \cite{stewart2024n} as draft) provide a statistically significant speedup over regular sequential decoding, as predicted by Lemma \ref{lem:num_evals}. (When decoding sequentially, the number of calls to the neural network is the same as the number of masked tokens.) \assd with paralllel sampling is the fastest method, and requires the least total NFEs. \assd with context n-gram is not that far behind: the intuition is that although it gives lower-quality drafts, it is very cheap. This phenomenon was also observed in similar studies for AR models \cite{stewart2024n}. However, only $1.15$ tokens get generated per iteration when using context n-gram, as opposed to $2.24$ tokens with parallel sampling as draft. %The wallclock speedup from any-order speculative decoding is slightly less than the reduction in number of neural network forward passes, since negligible extra computation comes from verification and resampling steps.

\subsection{Infilling Benchmark Tasks}\label{sec:infilling_benchmark}

\begin{table}[]
    \centering
    \begin{tabular}{l c c c c c c}
        \rowcolor{gray!35} \textbf{Model} & \textbf{Size} & \textbf{Tokens} & \multicolumn{2}{c}{\textbf{Infill 1/5}} & \multicolumn{2}{c}{\textbf{Infill 3/5}} \\
        \midrule
        \rowcolor{gray!20} & & & \textbf{\textit{ROUGE 1/2/L}} & \textbf{\textit{NFE}} & \textbf{\textit{ROUGE 1/2/L}} & \textbf{\textit{NFE}} \\
        \toprule
        GPT2-S & 127M & n/a & $9.5/0.4/8.7$ & $10.7 \pm 2.8$ & $13.5/0.6/10.2$ & $31.8 \pm 6.0$\\
        SEDD-S & 170M & 210B & $11.6/0.8/10.7$ & $32.0 \pm 0.0$ & $16.2/1.3/12.2$ & $64.0 \pm 0.0$ \\
        MDLM & 130M & 262B & $11.6/1.1/10.7$ & $32.0 \pm 0.0$ & $13.3/1.0/10.4$ & $64.0 \pm 0.0$\\
        DiffuGPT-S & 127M & 130B & $14.0/1.5/13.0$ & $32.0 \pm 0.0$ & $16.4/\mathbf{2.0}/\mathbf{14.2}$ & $64.0 \pm 0.0$ \\
        XLNet-OTS & \textit{\textbf{110M}} & \textit{\textbf{33B}} & $\mathbf{14.4}/\mathbf{1.7}/\mathbf{13.1}$ & $\mathbf{8.7} \pm 2.5$ & $7.7/0.6/6.4$ & $\mathbf{18.5} \pm 6.8$ \\
        \textit{XLNet-FT} & \textit{\textbf{110M}} & \textit{\textbf{12B}} & $13.1/1.1/12.0$ & $10.4 \pm 2.8$ & $\mathbf{18.0}/1.4/13.2$ & $\underline{30.4} \pm 6.0$ \\
    \bottomrule
    \end{tabular}
    \caption{\textbf{Performance on ROCStories Infilling:} We compare ROUGE ($\uparrow$) of various models on infilling ability. "Tokens" is the number of training tokens (excluding those for the pretrained initialization). In "Infill 1/5", sentences \{1, 2, 4, 5\} are input, and \{3\} is infilled. In "Infill 3/5", sentences \{1, 5\} are input, and \{2, 3, 4\} are infilled. We report $\mu \pm \sigma$ NFEs.}
    \label{tab:infill_results}
\end{table}

\noindent\textbf{Natural Language:} We also investigate whether \asarm can outperform discrete diffusion models and regular autoregressive models on infilling tasks, in which bidirectional context is key.
As such, we follow \cite{gong2024scaling}'s evaluation setup on the ROCStories dataset \cite{rocstories_mostafazadeh2016corpus}. We test on $1871$ short stories of five sentences each, and get $5$ completions (trials) per story. We blank out either the middle one or middle three sentences, and have the model predict the missing sentence(s). 

See Table \ref{tab:infill_results} for results. We see that XLNet is generally better than previous models of comparable size. This is encouraging, since it is the smallest model in the table, and used the fewest training resources.
Off-the-shelf (OTS) XLNet without any finetuning is the best at infilling a single missing sentence. Indeed, infilling one out of five sentences corresponds to a $\sim20\%$ masking ratio, which is roughly what it was trained on. Finetuning on a wider distribution of masking ratios following the objective in Equation \ref{eqn:loss}, splits finite model capacity among multiple tasks. %, which could understandably slightly degrade performance on a specific task. 
Unsurprisingly, when the task moves to infilling three out of five sentences, our finetuned (FT) XLNet clearly surpasses the off-the-shelf XLNet, and is comparable to DiffuGPT-S.

\begin{table}[]
    \centering
    \begin{tabular}{l r r r c}
        \rowcolor{gray!35} \textbf{Model} & \textbf{Size} & \textbf{Code Tokens} & \textbf{Lang. Tokens} & \textbf{Pass @ 1}\\
        \midrule
        XLNet-Code & \textbf{110M} & \textbf{15B} & \textbf{0B} & 38.59 \\
        DiffuLLaMA & 6738M & 19B & 46B & \textbf{40.68} \\
    \bottomrule
    \end{tabular}
    \caption{\textbf{Performance on HumanEval Infilling:} Comparison of code infilling abilities of XLNet finetuned on code versus DiffuLLaMA. We use HumanEval's single-line infilling task \cite{bavarian2022efficient}, evaluated by pass@1 (each attempt counts, instead of only the best attempt) on $5165$ trials. %($1033$ test cases times $5$ trials per case). %DiffuLLaMA results reprinted from \cite{gong2024scaling}.
    }
    \label{tab:code_infill_results}
\end{table}

\noindent\textbf{{Code Generation:}} Table \ref{tab:code_infill_results} shows that \asarm finetuned on code are competitive with models that are orders of magnitude larger. More details about this experiment are in Appendix \ref{subsec:code_gen}.

\section{Related Works}
%\se{related work is a bit too long, can be put in appendix and shorter version in main text}

\noindent\textbf{Any-Order Autoregressive Models:}
The defining characteristic of any-order autoregressive models is their ability to estimate the probability density of arbitrary marginal and/or conditional queries. This is an extension of vanilla autoregressive models' ability to estimate probability density. 
We primarily base our work off of XLNet \cite{yang2019xlnet} and MAC \cite{shih2022training}. From XLNet, we adopt the causal attention architecture: crucially, this allows us to estimate probability density in a parallelized single step, \textit{i.e.}, $O(1)$ time \cite{yang2019xlnet}. From MAC, we adopt recursive decomposition of queries on a binary lattice, which reduces the permutations needed to be learned from $N!$ to $2^N$ \cite{shih2022training}: this characteristic made MAC the first any-subset autoregressive model.

Although XLNet came in 2019, subsequent AO-ARMs \cite{hoogeboom2021autoregressive, strauss2021arbitrary, shih2022training} abandoned causal attention, in favor of full attention with absorbing state tokens to represent "masked" queries. %While this does not affect the correctness of joint density estimation, it 
This increases runtime of joint density estimation from $O(1)$ to $O(N)$: to satisfy the factorization rule, each token's density is estimated one-by-one so it conditions only on the preceding tokens.

\noindent\textbf{Discrete Diffusion:}
Discrete diffusion models data generation as the reversal of an inhomogeneous continuous-time Markov chain (CTMC). Empirically, the best performing CTMC is an absorbing state process. %, although it is possible to use other transition matrices. 
The advantages of discrete diffusion are that it can generate tokens in parallel, with arbitary prompt locations.
The seminal work is D3PM \cite{austin2021structured}. Later, LDR \cite{campbell2022continuous} formalized discrete diffusion under the lens of CTMC theory. More recently, SEDD \cite{lou2023discrete}, MDLM \cite{sahoo2024simple}, DiffuLLaMA \cite{gong2024scaling}, and LLaDA \cite{nie2025large} created formulations of discrete diffusion models that were competitive with or surpassing autoregressive models in performance.

One limitation of discrete diffusion models is their inability to do joint density estimation. %, which is partially related to the ELBO bound used to train them. 
Also, although tokens can be generated in parallel with large discretized timesteps, they are generated under the conditional independence assumption, which may not adhere to the limiting joint distribution. 
Furthermore, discrete diffusion architectures perform full attention across all tokens, even the masked ones \cite{lou2023discrete, sahoo2024simple}; due to this lack of causal attention masking, it is not straightforward to apply KV-caching speedup techniques.

\noindent\textbf{Speculative Decoding:}
Speculative decoding is an algorithm developed for autoregressive models that enables quick token generation from an inexpensive draft model. Via some techniques related to rejection sampling, the outputs provably adhere to the joint distribution of the target language model \cite{leviathan2023fast, chen2023accelerating}. 
Further related works are in Appendix \ref{sec:more_related_works}.

\section{Discussion, Limitations, Conclusion}

Our results suggest that the long-forgotten any-subset autoregressive models (\asarm) are a promising language modeling methodology. We theoretically and empirically show that the longstanding problem of principled parallel token generation from the joint distribution is easily solved once we combine \asarm' parallelized density estimation and generation capabilities with speculative decoding. Furthermore, benchmark tasks indicate that \asarm can match or surpass the currently dominant model designs.
Extensions include arbitrary-order KV caching (which is non-trivial, due to relative positional encodings \cite{shaw2018self}), adding Flash Attention \cite{dao2022flashattention}, investigating scaling laws, and designing more optimized architectures.
%We are considering multiple extensions. Firstly, arbitrary-order KV caching can reduce the number of FLOPs without changing the model outputs. Secondly, the model currently is not implemented with Flash Attention \cite{dao2022flashattention}. Finally, we hope to investigate scaling laws. We will release source code upon acceptance.

\section*{Acknowledgements}

\textcolor{blue}{This material is based upon work supported by the U.S. Department of Energy, Office of Science, Office of Advanced Scientific Computing Research, Department of Energy Computational Science Graduate Fellowship under Award Number DE-SC0025528 to Gabe Guo.
We thank Amil Merchant, Tristan Saidi, Luke Bailey, and Ajay Sridhar for helpful discussion.
}

\printbibliography

%%%%%%%%%%%%%%%%%%%%%%%%%%%%%%%%%%%%%%%%%%%%%%%%%%%%%%%%%%%%

\newpage

\newpage
\section*{NeurIPS Paper Checklist}

\begin{enumerate}

\item {\bf Claims}
    \item[] Question: Do the main claims made in the abstract and introduction accurately reflect the paper's contributions and scope?
    \item[] Answer: \answerYes{} % Replace by \answerYes{}, \answerNo{}, or \answerNA{}.
    \item[] Justification: The abstract and summary are essentially shortened versions of our paper, except without extensive technical details.
    \item[] Guidelines:
    \begin{itemize}
        \item The answer NA means that the abstract and introduction do not include the claims made in the paper.
        \item The abstract and/or introduction should clearly state the claims made, including the contributions made in the paper and important assumptions and limitations. A No or NA answer to this question will not be perceived well by the reviewers. 
        \item The claims made should match theoretical and experimental results, and reflect how much the results can be expected to generalize to other settings. 
        \item It is fine to include aspirational goals as motivation as long as it is clear that these goals are not attained by the paper. 
    \end{itemize}

\item {\bf Limitations}
    \item[] Question: Does the paper discuss the limitations of the work performed by the authors?
    \item[] Answer: \answerYes{} % Replace by \answerYes{}, \answerNo{}, or \answerNA{}.
    \item[] Justification: The last paragraph contains a list of future directions, which are things our study did not get to address.
    \item[] Guidelines:
    \begin{itemize}
        \item The answer NA means that the paper has no limitation while the answer No means that the paper has limitations, but those are not discussed in the paper. 
        \item The authors are encouraged to create a separate "Limitations" section in their paper.
        \item The paper should point out any strong assumptions and how robust the results are to violations of these assumptions (e.g., independence assumptions, noiseless settings, model well-specification, asymptotic approximations only holding locally). The authors should reflect on how these assumptions might be violated in practice and what the implications would be.
        \item The authors should reflect on the scope of the claims made, e.g., if the approach was only tested on a few datasets or with a few runs. In general, empirical results often depend on implicit assumptions, which should be articulated.
        \item The authors should reflect on the factors that influence the performance of the approach. For example, a facial recognition algorithm may perform poorly when image resolution is low or images are taken in low lighting. Or a speech-to-text system might not be used reliably to provide closed captions for online lectures because it fails to handle technical jargon.
        \item The authors should discuss the computational efficiency of the proposed algorithms and how they scale with dataset size.
        \item If applicable, the authors should discuss possible limitations of their approach to address problems of privacy and fairness.
        \item While the authors might fear that complete honesty about limitations might be used by reviewers as grounds for rejection, a worse outcome might be that reviewers discover limitations that aren't acknowledged in the paper. The authors should use their best judgment and recognize that individual actions in favor of transparency play an important role in developing norms that preserve the integrity of the community. Reviewers will be specifically instructed to not penalize honesty concerning limitations.
    \end{itemize}

\item {\bf Theory assumptions and proofs}
    \item[] Question: For each theoretical result, does the paper provide the full set of assumptions and a complete (and correct) proof?
    \item[] Answer: \answerYes{} % Replace by \answerYes{}, \answerNo{}, or \answerNA{}.
    \item[] Justification: See the appendix for proofs.
    \item[] Guidelines:
    \begin{itemize}
        \item The answer NA means that the paper does not include theoretical results. 
        \item All the theorems, formulas, and proofs in the paper should be numbered and cross-referenced.
        \item All assumptions should be clearly stated or referenced in the statement of any theorems.
        \item The proofs can either appear in the main paper or the supplemental material, but if they appear in the supplemental material, the authors are encouraged to provide a short proof sketch to provide intuition. 
        \item Inversely, any informal proof provided in the core of the paper should be complemented by formal proofs provided in appendix or supplemental material.
        \item Theorems and Lemmas that the proof relies upon should be properly referenced. 
    \end{itemize}

    \item {\bf Experimental result reproducibility}
    \item[] Question: Does the paper fully disclose all the information needed to reproduce the main experimental results of the paper to the extent that it affects the main claims and/or conclusions of the paper (regardless of whether the code and data are provided or not)?
    \item[] Answer: \answerYes{} % Replace by \answerYes{}, \answerNo{}, or \answerNA{}.
    \item[] Justification: We have comprehensive details in the appendix. We will also release code upon acceptance.
    \item[] Guidelines:
    \begin{itemize}
        \item The answer NA means that the paper does not include experiments.
        \item If the paper includes experiments, a No answer to this question will not be perceived well by the reviewers: Making the paper reproducible is important, regardless of whether the code and data are provided or not.
        \item If the contribution is a dataset and/or model, the authors should describe the steps taken to make their results reproducible or verifiable. 
        \item Depending on the contribution, reproducibility can be accomplished in various ways. For example, if the contribution is a novel architecture, describing the architecture fully might suffice, or if the contribution is a specific model and empirical evaluation, it may be necessary to either make it possible for others to replicate the model with the same dataset, or provide access to the model. In general. releasing code and data is often one good way to accomplish this, but reproducibility can also be provided via detailed instructions for how to replicate the results, access to a hosted model (e.g., in the case of a large language model), releasing of a model checkpoint, or other means that are appropriate to the research performed.
        \item While NeurIPS does not require releasing code, the conference does require all submissions to provide some reasonable avenue for reproducibility, which may depend on the nature of the contribution. For example
        \begin{enumerate}
            \item If the contribution is primarily a new algorithm, the paper should make it clear how to reproduce that algorithm.
            \item If the contribution is primarily a new model architecture, the paper should describe the architecture clearly and fully.
            \item If the contribution is a new model (e.g., a large language model), then there should either be a way to access this model for reproducing the results or a way to reproduce the model (e.g., with an open-source dataset or instructions for how to construct the dataset).
            \item We recognize that reproducibility may be tricky in some cases, in which case authors are welcome to describe the particular way they provide for reproducibility. In the case of closed-source models, it may be that access to the model is limited in some way (e.g., to registered users), but it should be possible for other researchers to have some path to reproducing or verifying the results.
        \end{enumerate}
    \end{itemize}

\item {\bf Open access to data and code}
    \item[] Question: Does the paper provide open access to the data and code, with sufficient instructions to faithfully reproduce the main experimental results, as described in supplemental material?
    \item[] Answer: \answerYes{} % Replace by \answerYes{}, \answerNo{}, or \answerNA{}.
    \item[] Justification: Data and code will be released upon acceptance. It currently contains identifying information about the authors, so it is not in the submission.
    \item[] Guidelines:
    \begin{itemize}
        \item The answer NA means that paper does not include experiments requiring code.
        \item Please see the NeurIPS code and data submission guidelines (\url{https://nips.cc/public/guides/CodeSubmissionPolicy}) for more details.
        \item While we encourage the release of code and data, we understand that this might not be possible, so “No” is an acceptable answer. Papers cannot be rejected simply for not including code, unless this is central to the contribution (e.g., for a new open-source benchmark).
        \item The instructions should contain the exact command and environment needed to run to reproduce the results. See the NeurIPS code and data submission guidelines (\url{https://nips.cc/public/guides/CodeSubmissionPolicy}) for more details.
        \item The authors should provide instructions on data access and preparation, including how to access the raw data, preprocessed data, intermediate data, and generated data, etc.
        \item The authors should provide scripts to reproduce all experimental results for the new proposed method and baselines. If only a subset of experiments are reproducible, they should state which ones are omitted from the script and why.
        \item At submission time, to preserve anonymity, the authors should release anonymized versions (if applicable).
        \item Providing as much information as possible in supplemental material (appended to the paper) is recommended, but including URLs to data and code is permitted.
    \end{itemize}

\item {\bf Experimental setting/details}
    \item[] Question: Does the paper specify all the training and test details (e.g., data splits, hyperparameters, how they were chosen, type of optimizer, etc.) necessary to understand the results?
    \item[] Answer: \answerYes{} % Replace by \answerYes{}, \answerNo{}, or \answerNA{}.
    \item[] Justification: See Appendix \ref{sec:additional_experimental_details}.
    \item[] Guidelines:
    \begin{itemize}
        \item The answer NA means that the paper does not include experiments.
        \item The experimental setting should be presented in the core of the paper to a level of detail that is necessary to appreciate the results and make sense of them.
        \item The full details can be provided either with the code, in appendix, or as supplemental material.
    \end{itemize}

\item {\bf Experiment statistical significance}
    \item[] Question: Does the paper report error bars suitably and correctly defined or other appropriate information about the statistical significance of the experiments?
    \item[] Answer: \answerYes{} % Replace by \answerYes{}, \answerNo{}, or \answerNA{}.
    \item[] Justification: See Table \ref{tab:speculative_decoding}.
    \item[] Guidelines:
    \begin{itemize}
        \item The answer NA means that the paper does not include experiments.
        \item The authors should answer "Yes" if the results are accompanied by error bars, confidence intervals, or statistical significance tests, at least for the experiments that support the main claims of the paper.
        \item The factors of variability that the error bars are capturing should be clearly stated (for example, train/test split, initialization, random drawing of some parameter, or overall run with given experimental conditions).
        \item The method for calculating the error bars should be explained (closed form formula, call to a library function, bootstrap, etc.)
        \item The assumptions made should be given (e.g., Normally distributed errors).
        \item It should be clear whether the error bar is the standard deviation or the standard error of the mean.
        \item It is OK to report 1-sigma error bars, but one should state it. The authors should preferably report a 2-sigma error bar than state that they have a 96\% CI, if the hypothesis of Normality of errors is not verified.
        \item For asymmetric distributions, the authors should be careful not to show in tables or figures symmetric error bars that would yield results that are out of range (e.g. negative error rates).
        \item If error bars are reported in tables or plots, The authors should explain in the text how they were calculated and reference the corresponding figures or tables in the text.
    \end{itemize}

\item {\bf Experiments compute resources}
    \item[] Question: For each experiment, does the paper provide sufficient information on the computer resources (type of compute workers, memory, time of execution) needed to reproduce the experiments?
    \item[] Answer: \answerYes{} % Replace by \answerYes{}, \answerNo{}, or \answerNA{}.
    \item[] Justification: See Appendix \ref{sec:additional_experimental_details}.
    \item[] Guidelines:
    \begin{itemize}
        \item The answer NA means that the paper does not include experiments.
        \item The paper should indicate the type of compute workers CPU or GPU, internal cluster, or cloud provider, including relevant memory and storage.
        \item The paper should provide the amount of compute required for each of the individual experimental runs as well as estimate the total compute. 
        \item The paper should disclose whether the full research project required more compute than the experiments reported in the paper (e.g., preliminary or failed experiments that didn't make it into the paper). 
    \end{itemize}
    
\item {\bf Code of ethics}
    \item[] Question: Does the research conducted in the paper conform, in every respect, with the NeurIPS Code of Ethics \url{https://neurips.cc/public/EthicsGuidelines}?
    \item[] Answer: \answerYes{} % Replace by \answerYes{}, \answerNo{}, or \answerNA{}.
    \item[] Justification: Experiments were purely computational and conducted based on publicly available data.
    \item[] Guidelines:
    \begin{itemize}
        \item The answer NA means that the authors have not reviewed the NeurIPS Code of Ethics.
        \item If the authors answer No, they should explain the special circumstances that require a deviation from the Code of Ethics.
        \item The authors should make sure to preserve anonymity (e.g., if there is a special consideration due to laws or regulations in their jurisdiction).
    \end{itemize}

\item {\bf Broader impacts}
    \item[] Question: Does the paper discuss both potential positive societal impacts and negative societal impacts of the work performed?
    \item[] Answer: \answerNo{} % Replace by \answerYes{}, \answerNo{}, or \answerNA{}.
    \item[] Justification: The negative societal impacts are the same as the negative societal impacts of any generative language model: disinformation, unsafe code generation, etc. The positive societal impacts are also generic: accelerated writing efficiency, etc.
    \item[] Guidelines:
    \begin{itemize}
        \item The answer NA means that there is no societal impact of the work performed.
        \item If the authors answer NA or No, they should explain why their work has no societal impact or why the paper does not address societal impact.
        \item Examples of negative societal impacts include potential malicious or unintended uses (e.g., disinformation, generating fake profiles, surveillance), fairness considerations (e.g., deployment of technologies that could make decisions that unfairly impact specific groups), privacy considerations, and security considerations.
        \item The conference expects that many papers will be foundational research and not tied to particular applications, let alone deployments. However, if there is a direct path to any negative applications, the authors should point it out. For example, it is legitimate to point out that an improvement in the quality of generative models could be used to generate deepfakes for disinformation. On the other hand, it is not needed to point out that a generic algorithm for optimizing neural networks could enable people to train models that generate Deepfakes faster.
        \item The authors should consider possible harms that could arise when the technology is being used as intended and functioning correctly, harms that could arise when the technology is being used as intended but gives incorrect results, and harms following from (intentional or unintentional) misuse of the technology.
        \item If there are negative societal impacts, the authors could also discuss possible mitigation strategies (e.g., gated release of models, providing defenses in addition to attacks, mechanisms for monitoring misuse, mechanisms to monitor how a system learns from feedback over time, improving the efficiency and accessibility of ML).
    \end{itemize}
    
\item {\bf Safeguards}
    \item[] Question: Does the paper describe safeguards that have been put in place for responsible release of data or models that have a high risk for misuse (e.g., pretrained language models, image generators, or scraped datasets)?
    \item[] Answer: \answerNA{} % Replace by \answerYes{}, \answerNo{}, or \answerNA{}.
    \item[] Justification: This is a proof-of-concept. The work has not been scaled to a level where it could be harmful.
    \item[] Guidelines:
    \begin{itemize}
        \item The answer NA means that the paper poses no such risks.
        \item Released models that have a high risk for misuse or dual-use should be released with necessary safeguards to allow for controlled use of the model, for example by requiring that users adhere to usage guidelines or restrictions to access the model or implementing safety filters. 
        \item Datasets that have been scraped from the Internet could pose safety risks. The authors should describe how they avoided releasing unsafe images.
        \item We recognize that providing effective safeguards is challenging, and many papers do not require this, but we encourage authors to take this into account and make a best faith effort.
    \end{itemize}

\item {\bf Licenses for existing assets}
    \item[] Question: Are the creators or original owners of assets (e.g., code, data, models), used in the paper, properly credited and are the license and terms of use explicitly mentioned and properly respected?
    \item[] Answer: \answerYes{} % Replace by \answerYes{}, \answerNo{}, or \answerNA{}.
    \item[] Justification: We build our work upon open-source, and cite all relevant papers.
    \item[] Guidelines:
    \begin{itemize}
        \item The answer NA means that the paper does not use existing assets.
        \item The authors should cite the original paper that produced the code package or dataset.
        \item The authors should state which version of the asset is used and, if possible, include a URL.
        \item The name of the license (e.g., CC-BY 4.0) should be included for each asset.
        \item For scraped data from a particular source (e.g., website), the copyright and terms of service of that source should be provided.
        \item If assets are released, the license, copyright information, and terms of use in the package should be provided. For popular datasets, \url{paperswithcode.com/datasets} has curated licenses for some datasets. Their licensing guide can help determine the license of a dataset.
        \item For existing datasets that are re-packaged, both the original license and the license of the derived asset (if it has changed) should be provided.
        \item If this information is not available online, the authors are encouraged to reach out to the asset's creators.
    \end{itemize}

\item {\bf New assets}
    \item[] Question: Are new assets introduced in the paper well documented and is the documentation provided alongside the assets?
    \item[] Answer: \answerYes{} % Replace by \answerYes{}, \answerNo{}, or \answerNA{}.
    \item[] Justification: We have instructions on how to run the code. But again, the code is not released yet, since it contains some identifying author information.
    \item[] Guidelines:
    \begin{itemize}
        \item The answer NA means that the paper does not release new assets.
        \item Researchers should communicate the details of the dataset/code/model as part of their submissions via structured templates. This includes details about training, license, limitations, etc. 
        \item The paper should discuss whether and how consent was obtained from people whose asset is used.
        \item At submission time, remember to anonymize your assets (if applicable). You can either create an anonymized URL or include an anonymized zip file.
    \end{itemize}

\item {\bf Crowdsourcing and research with human subjects}
    \item[] Question: For crowdsourcing experiments and research with human subjects, does the paper include the full text of instructions given to participants and screenshots, if applicable, as well as details about compensation (if any)? 
    \item[] Answer: \answerNo{} % Replace by \answerYes{}, \answerNo{}, or \answerNA{}.
    \item[] Justification: We don't do human experiments.
    \item[] Guidelines:
    \begin{itemize}
        \item The answer NA means that the paper does not involve crowdsourcing nor research with human subjects.
        \item Including this information in the supplemental material is fine, but if the main contribution of the paper involves human subjects, then as much detail as possible should be included in the main paper. 
        \item According to the NeurIPS Code of Ethics, workers involved in data collection, curation, or other labor should be paid at least the minimum wage in the country of the data collector. 
    \end{itemize}

\item {\bf Institutional review board (IRB) approvals or equivalent for research with human subjects}
    \item[] Question: Does the paper describe potential risks incurred by study participants, whether such risks were disclosed to the subjects, and whether Institutional Review Board (IRB) approvals (or an equivalent approval/review based on the requirements of your country or institution) were obtained?
    \item[] Answer: \answerNA{} % Replace by \answerYes{}, \answerNo{}, or \answerNA{}.
    \item[] Justification: No human experiments.
    \item[] Guidelines:
    \begin{itemize}
        \item The answer NA means that the paper does not involve crowdsourcing nor research with human subjects.
        \item Depending on the country in which research is conducted, IRB approval (or equivalent) may be required for any human subjects research. If you obtained IRB approval, you should clearly state this in the paper. 
        \item We recognize that the procedures for this may vary significantly between institutions and locations, and we expect authors to adhere to the NeurIPS Code of Ethics and the guidelines for their institution. 
        \item For initial submissions, do not include any information that would break anonymity (if applicable), such as the institution conducting the review.
    \end{itemize}

\item {\bf Declaration of LLM usage}
    \item[] Question: Does the paper describe the usage of LLMs if it is an important, original, or non-standard component of the core methods in this research? Note that if the LLM is used only for writing, editing, or formatting purposes and does not impact the core methodology, scientific rigorousness, or originality of the research, declaration is not required.
    %this research? 
    \item[] Answer: \answerYes{} % Replace by \answerYes{}, \answerNo{}, or \answerNA{}.
    \item[] Justification: The whole paper is about language models.
    \item[] Guidelines:
    \begin{itemize}
        \item The answer NA means that the core method development in this research does not involve LLMs as any important, original, or non-standard components.
        \item Please refer to our LLM policy (\url{https://neurips.cc/Conferences/2025/LLM}) for what should or should not be described.
    \end{itemize}

\end{enumerate}

\newpage

\appendix

\section{Proofs}\label{sec:proofs}

Line numbers refer to Algorithm \ref{alg:any_order_speculative}.

\textbf{Lemma \ref{lem:first_token_accept}.} \textit{The first token speculated in each loop iteration will always be accepted.}

\begin{proof}
When $i = n$ on Line \ref{line:oracle_loop}, 
\begin{align}
    q_{\sigma(i)} &= p(\tilde{x}_{\sigma(i)} | \mathbf{x}_{\sigma(<n)}, \mathbf{\tilde{x}}_{\sigma[n:i)})  \\
    &= p(\tilde{x}_{\sigma(i)} | \mathbf{x}_{\sigma(<n)}, \mathbf{\tilde{x}}_{\sigma[n:n)}) \\
    &= p(\tilde{x}_{\sigma(i)} | \mathbf{x}_{\sigma(<n)}) \\
    &= p_{\sigma(i)}.
\end{align}
So, $r < \frac{q_{\sigma(i)}}{p_{\sigma(i)}} = 1$ on Line \ref{line:accept}, and thus, when $i = n$, $\tilde{x}_{\sigma(n)}$ will always be accepted on Line \ref{line:accept_proposal_token}.
\end{proof}

\textbf{Theorem \ref{lem:num_evals}.} \textit{Algorithm \ref{alg:any_order_speculative} requires no more than $N - m$ total function evaluations of $p(\cdot | \cdot)$. That is, there will never be more calls to a neural network than the number of tokens decoded.}

\begin{proof}
The idea underlying this proof is that each iteration of the while loop on Line \ref{line:main_loop} has two function evaluations.
The first function evaluation is on Lines \ref{line:speculate_loop} to \ref{line:end_speculate_loop}, to speculate tokens. The second function evaluation is on Lines \ref{line:oracle_loop} to \ref{line:end_oracle}, to evaluate the oracle density.

We then just need to show that at least two tokens will be decoded, making for minimum one token generated per function evaluation. 

By Lemma \ref{lem:first_token_accept}, the first token (at $i = n$) is always accepted. So, we can move on to $i = n + 1$.

When $i = n + 1$, regardless of whether we accept or reject $\tilde{x}_{\sigma(n+1)}$, we will still obtain a value for $x_{\sigma(n)}$, whether it is $\tilde{x}_{\sigma(n)}$ on Line \ref{line:accept_proposal_token} or a resampling from the adjusted $()_{+}$ distribution on Line \ref{line:resample} (which does not require additional calls to $p(\cdot | \cdot)$, as it makes use of already calculated distributions from Lines \ref{line:speculate_iter} and \ref{line:gt_density}).

So, on each loop iteration, we are guaranteed to get token values for at least $x_{\sigma(n)}, x_{\sigma(n+1)}$, thereby giving us the requisite two tokens.

On the last loop iteration, if $n = N - 1$, there is an edge case where only one token can be decoded (Line \ref{line:early_stop}). But, as previously shown in Lemma \ref{lem:first_token_accept}, this speculated token is mathematically guaranteed to be accepted, since it was the first and only one to be speculated. So, we can forgo the verification step for this final loop iteration. Thus, to generate this final token, we still only need one function evaluation, maintaining the lower bound of one token generated per function evaluation.
\end{proof}

% The upshot of Lemma \ref{lem:num_evals} is that we should always set $k > 2$, as it is always guaranteed that at least two tokens will get generated at each loop iteration.

\textbf{Theorem \ref{thm:correct_dist}.} \textit{Algorithm \ref{alg:any_order_speculative} produces samples from the true joint distribution $p(\mathbf{x}_{\sigma(\geq m)} | \mathbf{x}_{\sigma(<m)})$.}
\begin{proof}

The only differences between our algorithm and speculative decoding are:
\begin{enumerate}
    \item The use of $p$ as its own speculator. This is valid, because speculative decoding is agnostic to the speculator function, as long as it gives probability density estimates.
    \item The omission of the verification step when decoding the last token $n = T - 1$ on Line \ref{line:early_stop}. This is justified by Lemma \ref{lem:first_token_accept}, which shows that when $i = n = T - 1$, $\tilde{x}_{\sigma(n)}$ would always meet the acceptance criteria if it were to go through the verification on Line \ref{line:accept}.
    \item We do not sample an extra token if all of the speculations are accepted at the end of each loop iteration. However, this does not affect the distribution of the other tokens outputted. We skip this extra token sampling because, as shown in Lemma \ref{lem:first_token_accept}, the first token speculated in the next iteration will always be accepted, which ``makes up" for the omission of the extra sample. 
\end{enumerate}

Other than these inconsequential differences, this algorithm is simply speculative decoding, but with a different ordering than $\sigma = [0, 1, 2, \ldots, N - 1]$. We can permute/sort $\mathbf{x}$ in increasing order of $\sigma(i)$, and define as the result $\mathbf{y}$. Then, our algorithm is mathematically equivalent to speculative decoding with the inputs $\mathbf{y}$ as the sequence, and the draft model the same as the target model. The rest of the proof of correctness of our algorithm is therefore the same as in \cite{leviathan2023fast, chen2023accelerating}. For completeness, we reproduce their proof in Appendix \ref{sec:modified_rejection_sampling}, with some additional commentary.
\end{proof}

\section{Modified Rejection Sampling Correctness}\label{sec:modified_rejection_sampling}

We provide an extended proof of correctness for Theorem \ref{thm:correct_dist}, as it relates to the modified rejection sampling step in Algorithm \ref{alg:any_order_speculative}. Note that the main ideas and layout of this proof are copied from \cite{chen2023accelerating} and \cite{leviathan2023fast}. We only include it here for the reader's convenience and consistency with our notation.

\begin{proof}
WLOG, consider the random variable (token) $x_{\sigma(i)}$, returned by Algorithm \ref{alg:any_order_speculative} (\textit{\assd}). Consider the iteration where $n \leq i < n + k \leq N$, \textit{i.e.}, the iteration that generates $x_{\sigma(i)}$.

For the algorithm to be correct, we desire that $\mathbb{P}(x_{\sigma(i)} = x) = p(x_{\sigma(i)}=x | \mathbf{x}_{\sigma(<i)})$, where $\mathbb{P}(x_{\sigma(i)} = x)$ is the probability that \textit{\assd} generates token value $x$ at position $\sigma(i)$, and $p(x_{\sigma(i)}=x | \mathbf{x}_{\sigma(<i)})$ is the probability that regular sequential decoding would generate $x$ at position $\sigma(i)$. In other words, the distribution of \textit{\assd} is the same as in sequential decoding. Note that $\mathbf{x}_{\sigma(<i)} = \mathbf{x}_{\sigma(<n)} \oplus \mathbf{\tilde{x}}_{\sigma[n:i)}$ in both sequential decoding and our algorithm, because they both use generations from previous iterations (in addition to the prompt) as conditioning for future iterations.

Let $\tilde x$ (shorthand for $\tilde x_{\sigma(i)}$) be the token value generated from the conditionally independent draft in Line \ref{line:speculate_iter}. When $x_{\sigma(i)} = x$ is true, there are two mutually exclusive and collectively exhaustive possibilities (from the if-else statement): (1) $\tilde x$ was accepted, and $\tilde x = x$ (2) $\tilde x$ was rejected, and $x$ was the result of the resampling.

\begin{equation}\label{eqn:total_prob}
\begin{split}
    \mathbb{P}(x_{\sigma(i)} = x) &= \mathbb{P}(\tilde x \text{ accepted}, \tilde x = x) + \mathbb{P}(x_{\sigma(i)} = x, \tilde x \text{ rejected}) \\
    &= \mathbb{P}(\tilde x \text{ accepted } | \tilde x = x)\mathbb{P}(\tilde x = x) + \mathbb{P}(\tilde x \text{ rejected}) \mathbb{P}(x_{\sigma(i)} = x | \tilde x \text{ rejected})
\end{split}
\end{equation}

Analyzing the first term, we look to the "if" (accept) clause on Line \ref{line:accept}:
\begin{equation}\label{eqn:accept_proof}
\begin{split}
    \mathbb{P}(\tilde x \text{ accepted } | \tilde x = x)P(\tilde x = x) &= \text{min}\left(1, \frac{q_{\sigma(i)}(x)}{p_{\sigma(i)}(x)}\right) p_{\sigma(i)}(x) \\
    &= \text{min}\left(p_{\sigma(i)}(x), q_{\sigma(i)}(x)\right)
\end{split}
\end{equation}
Here, $q_{\sigma(i)}(x) = p(x_{\sigma(i)} = x | \mathbf{x}_{\sigma(<n)}, \mathbf{\tilde x}_{\sigma[n:i)})$, \textit{i.e.}, the oracle density from Line \ref{line:gt_density}; and $p_{\sigma(i)}(x) = p(x_{\sigma(i)} = x | \mathbf{x}_{\sigma(<n)})$, \textit{i.e.}, the speculator density from Line \ref{line:independent_density}.

Regarding the second term, we look at the "else" (reject) clause on Line \ref{line:reject}. We analyze each item in the product separately, as the expressions are long:
\begin{equation}\label{eqn:reject_prob}
\begin{split}
    \mathbb{P}(\tilde x \text{ rejected}) &= 
    1 - \mathbb{P}(\tilde x \text{ accepted}) \\
    &= 1 - \sum_{x'}\mathbb{P}(x_{\sigma(i)} = x', \tilde x \text{ accepted}) \\
    &= 1 - \sum_{x'}\text{min}\left(p_{\sigma(i)}(x'), q_{\sigma(i)}(x')\right) \\
    &= \sum_{x'}q_{\sigma(i)}(x') - \sum_{x'}\text{min}\left(p_{\sigma(i)}(x'), q_{\sigma(i)}(x')\right) \\
    &= \sum_{x'}q_{\sigma(i)}(x') - \text{min}\left(p_{\sigma(i)}(x'), q_{\sigma(i)}(x')\right) \\
    &= \sum_{x'}\text{max}\left(q_{\sigma(i)}(x') - p_{\sigma(i)}(x'), q_{\sigma(i)}(x') - q_{\sigma(i)}(x')\right) \\
    &= \sum_{x'}\text{max}\left(q_{\sigma(i)}(x') - p_{\sigma(i)}(x'), 0\right) \\
\end{split}
\end{equation}

From Line \ref{line:resample}:

\begin{equation}\label{eqn:get_x_cond_reject_prob}
\begin{split}
    \mathbb{P}(x_{\sigma(i)} = x | \tilde x \text{ rejected}) &= \left(p(x_{\sigma(i)} = x | \mathbf{x}_{\sigma(<n)}, \mathbf{\tilde{x}}_{\sigma[n:i)}) - p(x_{\sigma(i)} = x | \mathbf{x}_{\sigma(<n)})\right)_{+} \\
    &= \left(q_{\sigma(i)}(x) - p_{\sigma(i)}(x)\right)_{+} \\
    &= \frac{\text{max}\left(q_{\sigma(i)}(x) - p_{\sigma(i)}(x), 0\right)}{\sum_{x'}\text{max}\left(q_{\sigma(i)}(x') - p_{\sigma(i)}(x'), 0\right)}
\end{split}
\end{equation}

Following Equations \ref{eqn:reject_prob} and \ref{eqn:get_x_cond_reject_prob}, we have
\begin{equation}
\begin{split}
    \mathbb{P}(\tilde x \text{ rejected}) \mathbb{P}(x_{\sigma(i)} = x | \tilde x \text{ rejected}) &= \text{max}\left(q_{\sigma(i)}(x) - p_{\sigma(i)}(x), 0\right)
\end{split}
\end{equation}

Putting it all together back into Equation \ref{eqn:total_prob}, 
\begin{equation}\label{eqn:prob_x_proved}
\begin{split}
    \mathbb{P}(x_{\sigma(i)} = x) &= \text{min}\left(p_{\sigma(i)}(x), q_{\sigma(i)}(x)\right) + \text{max}\left(q_{\sigma(i)}(x) - p_{\sigma(i)}(x), 0\right) \\
    &= q_{\sigma(i)}(x) \\
    &= p(x_{\sigma(i)} = x | \mathbf{x}_{\sigma(<n)}, \mathbf{\tilde x}_{\sigma[n:i)}) \\
    &= p(x_{\sigma(i)} = x | \mathbf{x}_{\sigma(<i)})
\end{split}
\end{equation}

The last line of Equation \ref{eqn:prob_x_proved} is true, because to have gotten to index $i$ in the accept-reject loop (Line \ref{line:accept_reject_loop}), we must have accepted $\mathbf{\tilde x}_{\sigma[n:i)}$, \textit{i.e.}, $\mathbf{x}_{\sigma[n:i)} = \mathbf{\tilde x}_{\sigma[n:i)}$. Therefore, we have shown that \textit{\assd} gives the same per-token distribution as in sequential decoding.
Induction over $i \in [m: N)$ will easily show that \textit{\assd} gives the correct joint distribution.

\end{proof}

\section{Causal-Like Attention Masking}\label{sec:extra_causal_attention_masking}

We provide further discussion on the masking introduced in Section \ref{subsec:ao_arm_density}.

In particular, there is a subtlety in attention mechanisms: even if a token does not attend to itself, it still "sees" its own content because it constructs a query representation to compare against the keys. At first glance, this seemingly breaks the ban (Equation \ref{eqn:attn}) on a token attending to itself. The solution is to separate the positional and content information into two streams. The positional stream (which only has positional representations as input) is used to calculate the queries, \textit{i.e.}, row index of the attention map. The content stream (which has token value information) is used to calculate the key/value information, \textit{i.e.}, column index of the attention map. Thus, the evaluation of the positional queries does not "cheat" by looking at the ground truth content at its own position. Conceptually, this can be said to be more like cross-attention than the self-attention common in discrete diffusion models \cite{austin2021structured}. %For a more detailed treatment of this architectural design, we refer the reader to Figure 1 in the XLNet paper \cite{yang2019xlnet}.

In contrast, architectures commonly used for discrete diffusion models \cite{lou2023discrete, deschenaux2024beyond, sahoo2024simple, austin2021structured} and other any-order autoregressive models \cite{shih2022training, hoogeboom2021autoregressive} do not support this kind of masking. In these architectures,
\begin{equation}\label{eqn:bad_masking}
    \forall i, j: A_{\sigma(i), \sigma(j)} = 1,
\end{equation}
which means that every token is allowed to attend to every token including itself when calculating its probability. From a probabilistic graphical modeling perspective, the outputted probabilities very roughly correspond to $\text{log } p(x_{\sigma(i)} | \mathbf{x})$, which would \textit{not} give us principled estimates on the already-visible tokens.

\section{Additional Experimental Details}\label{sec:additional_experimental_details}

\subsection{Dataset}

We finetune on the OpenWebText dataset \cite{Gokaslan2019OpenWeb}. Following \cite{sahoo2024simple}, we pack the sequences together, and split them into chunks of $512$ tokens, based on XLNet's case-sensitive tokenizer with a vocabulary of $32,000$ possible tokens. We have separator tokens to delineate the start of a new document.

\subsection{Mask Distribution}\label{subsec:mask_distribution}

%For our work, we want to create a versatile model, so we adopt a wide range of prompt lengths and sample $m \sim \mathcal{U}[0.01 * N, 0.85 * N]$, where $m$ is rounded to an integer. 

We are interested in generating text from near-scratch, so we train a model where $m \sim \mathcal{U}[0.01 * N, 0.10 * N]$. To get $\sigma$, we first sample $\sigma_\text{pre} \sim \mathcal{U}(S_N)$, where $S_N$ represents all the possible permutations of the integers from $[0, N)$. But, remembering Section \ref{subsec:efficient_mask}'s efficient masking protocol, we sort each of $\sigma_\text{pre}(<m)$ and $\sigma_\text{pre}(\geq m)$ in ascending order to get $\sigma$. This eliminates the ambiguity in paths the model has to learn to calculate a given joint, while still maintaining which positions are in the prompt and which need to be predicted.

We also use a similar low-discrepancy sampler as \cite{sahoo2024simple} to reduce variance among prompt lengths within a batch.

% Within each batch on a device, we give every item the same $t$. We use a similar low-discrepancy sampler to \cite{sahoo2024simple}, where each device (in distributed training) has a bin to sample the mask rate from. Mathematically, $t_i \sim \mathcal{U}\left[t_\text{min} + \frac{(i)(t_\text{max} - t_\text{min})}{d}, t_\text{min} + \frac{(i+1)(t_\text{max} - t_\text{min})}{d}\right]$, where $i$ is the 0-indexed device number, and $d$ is the number of devices. (Aggregated across all devices, this is equivalent to $t \sim \mathcal{U}[t_\text{min}, t_\text{max}]$.)

\subsection{Training Hyperparameters}\label{subsec:hyperparameters}

%For the results in Table \ref{tab:speculative_decoding}, 
In finetuning XLNet, we use maximum learning rate $10^{-4}$ and batch size $320$ ($16$ per device, $4$ accumulations, $5$ devices). We have linear learning rate warmup for $5000$ steps, and linear learning rate decay for $70,000$ additional steps, making for a total of $2.4 * 10^7$ samples seen. (Notably, this is far fewer than the $2.56 * 10^8$ samples that DiffuGPT-S saw \cite{gong2024scaling}.)  We start at $15\%$ masking rate, then linearly increase the minimum masking rate to $90\%$ and the maximum masking rate to $99\%$ over $5000$ steps.

% For the results in Table \ref{tab:infill_results}, we use maximum learning rate $1.4 * 10^{-4}$, batch size $448$ ($16$ per device, $4$ accumulations, $7$ devices), and early stop training after $40,000$ steps. We have linear learning rate warmup for $5000$ steps, and linear learning rate decay scheduled until the model sees $2 * 10^7$ samples (equivalent to $44,642$ steps). We have minimum masking rate $15\%$. We linearly increase the maximum masking rate from $15\%$ to $99\%$ over $5000$ steps. 

We train on NVIDIA RTX A6000 Ada devices, which have 48 GB of RAM each. Training lasts around four days. We calculate validation loop metrics every $500$ steps on $64$ samples. We use the AdamW \cite{loshchilov2017fixing, kingma2014adam} optimizer. 

We also tried training for more epochs, but we found that performance saturated early on.

\subsection{Metrics}

We are focused on generation quality. To that end, we have two metrics. 

To quantify how likely the generated outputs of our model are, we calculate generative perplexity, where lower values are better. It is calculated as 
\begin{equation}
    \text{PPL}_\text{gen} = \left(e^{\sum_{i}\text{log }q(x_{i} | \mathbf{x}_{0 \ldots i-1})}\right)^{-1/N},
\end{equation}
where $q$ is an oracle language model (we use GPT-2 Large \cite{radford2019language}), and $\mathbf{x}$ is a generated sequence.

To quantify the diversity in tokens, we calculate Shannon entropy \cite{shannon1948mathematical}, where higher values are better. The formula is
\begin{equation}
H(\mathbf{x}) = -\sum_{x_i \in \mathbf{x}}\text{log}_{2}(p(x_i))p(x_i),
\end{equation}
where $p(x_i) = \frac{|\{j | \mathbf{x}_j = x_i\}|}{|\mathbf{x}|}$, \textit{i.e.}, the frequency of a token in the sequence. 

Typically, there is a trade-off between generative perplexity and Shannon entropy -- highly repetitive sequences of common words like "a" have low generative perplexity, but also low entropy. Random sequences of gibberish have high entropy and high generative perplexity. We want to generate sequences with high entropy and low generative perplexity.

\subsection{Speculative Decoding Variants}\label{subsec:speculative_decoding_experiment_details}

We compare sequential decoding to two variants of \assd: the first is described in Algorithm \ref{alg:any_order_speculative}. The second uses context-based n-grams, which were initially proposed for left-to-right speculative decoding \cite{stewart2024n}, although they easily fit into our framework. To adopt it to arbitrary order, we replace the speculator $p(\cdot | \cdot)$ in Lines \ref{line:speculate_loop}-\ref{line:end_speculate_loop} of Algorithm \ref{alg:any_order_speculative} with a context-based n-gram model $c(\cdot | \cdot)$, as in Algorithm \ref{alg:context_based_n_gram_speculation}. There, $c(a | b)$ is the probability over the partially decoded sequence that a bigram starting in $b$ ends in $a$, as in Equation \ref{eqn:context_n_gram}:
\begin{equation}\label{eqn:context_n_gram}
    c(a | b) = \frac{\sum_{i, x_i \neq \texttt{MASK}, x_{i+1} \neq \texttt{MASK}}\mathbbm{1}[\mathbf{x}_{i, i+1} = (a, b)]}{\sum_{j, x_j \neq \texttt{MASK}}\mathbbm{1}[x_{j} = b]},
\end{equation}
We initialize $c(\cdot | \cdot)$ by sweeping over the prompt, then update it iteratively, as the sequence is decoded. 

\begin{algorithm}
\caption{Speculation with Context-Based N-Grams}\label{alg:context_based_n_gram_speculation}
\LinesNumbered
\SetAlgoLined
\KwIn{same as Algorithm \ref{alg:any_order_speculative}}
\KwOut{see Algorithm \ref{alg:any_order_speculative}}
\For{$i \in [n:t)$}{
    \eIf{$x_{\sigma(i) - 1} \neq \texttt{MASK}$}{\label{line:if_ngram}
        $x_\text{cond} \leftarrow x_{\sigma(i) - 1}$ \\
    }{
        $x_\text{cond} \leftarrow \tilde{x}_{\sigma(i) - 1}$ \\
    }\label{line:end_if_ngram}
    $\tilde{x}_{\sigma(i)} \sim c(\cdot | x_\text{cond})$ // sample from partially conditioned distribution \\
    $p_{\sigma(i)} \leftarrow c(\tilde{x}_{\sigma(i)} | x_\text{cond})$ // get partially conditioned density \\
}
\end{algorithm}

\begin{theorem}
When $i \geq 1$, Algorithm \ref{alg:context_based_n_gram_speculation} always sets $x_\text{cond}$ to a valid non-\texttt{MASK} value.
\end{theorem}

\begin{proof}
We conduct strong induction on $i \geq m$, where $m$ is the given prompt length.

For a given value of $i$, we have two cases:
\begin{enumerate}
    \item $x_{\sigma(i) - 1} \neq \texttt{MASK}$. This case is trivial.
    \item $x_{\sigma(i) - 1} = \texttt{MASK}$. This reduces to showing that $\tilde{x}_{\sigma(i) - 1} \neq \texttt{MASK}$. Firstly, there exists some $j$ such that $\sigma(j) = \sigma(i) - 1$, which trivially means $\sigma(j) < \sigma(i)$. From Equation \ref{eqn:left_to_right_forcing}, this is equivalent to saying $j < i$. By the inductive hypothesis, $\tilde{x}_{\sigma(j)} \neq \texttt{MASK}$ should have been speculated for all $j < i$.
\end{enumerate}
\end{proof}

We note that context-based n-grams lack the guarantees of Lemma \ref{lem:first_token_accept}, so they could increase the total number of function evaluations (\asarm + n-gram draft) beyond the sequence length. However, n-grams are typically cheap to evaluate, which could make up for the increased function evaluations.

\subsection{Infilling Benchmark}

%For the diffusion baselines, we set the number of diffusion steps to be the smallest power of two that is more than double the number of tokens to predict, to ensure that they do not get approximation error in estimating the transition probabilities. 
We run $5$ trials over the dataset of $1871$ stories, making for $9355$ samples per masking level (short or long) per model. 
The models compared are: GPT-2 \cite{radford2019language}, MDLM \cite{sahoo2024simple}, SEDD \cite{lou2023discrete}, DiffuGPT \cite{gong2024scaling}, XLNet-OTS (off-the-shelf) \cite{yang2019xlnet}, and XLNet-Finetuned (finetuned by us). We do not compare against models like LLaDA \cite{nie2025large} because they have access to significantly more training budget than us, not to mention their orders-of-magnitude larger model size.

For the XLNet models, we use \assd. We set $k=15$ for our speculative decoding. 
Without speculative decoding on XLNet, "Infill 1/5" requires $10.9 \pm 2.9$, and "Infill 3/5" requires $32.4 \pm 6.2$ evaluations.
For GPT, we use sequential decoding, since it cannot be its own speculator.
Following \cite{gong2024scaling}, we only give GPT the left conditioning, as it is not straightforward to give it rightward conditioning without instruction-tuning. 
This experiment can be run on a 16GB GPU.

\subsection{Code Generation}\label{subsec:code_gen}

We finetune XLNet on 14.7 billion tokens from Starcoder's Python data \cite{li2023starcoder}. This corresponds to $75,000$ training steps with a batch size of $384$ ($16$ per device, $4$ accumulation steps, $6$ GPUs) of $512$ tokens each. We started from learning rate $0$, warmed up for $20,000$ steps to learning rate $1.2 * 10^{-4}$, then scheduled to decay linearly for $63,333$ steps (which corresponds to $32 * 10^6$ total samples seen), although the run crashed slightly earlier than this. We start from a $5\%$ masking rate, then warm up over $20,000$ steps to a $15\%$ minimum masking rate, and a $99\%$ maximum masking rate, sampled uniformly.

Since the XLNet tokenizer does not support whitespaces, we replace whitespace tokens with special characters during training (\textbackslash n : <cls>, \textbackslash t : <sep>, "\_ \_" : <unk>, "\_ \_ \_" : <pad>), where \_ is the space. When generating code, we reverse the special character mapping and truncate the leading space from the generated lines, since the tokenizer by default inserts a leading space to each word.

We evalute on the HumanEval single-line infilling task on $1033$ Python test cases, following \cite{gong2024scaling}. We generated $5$ completions for each test case, for a sample size of $5165$. We evaluate with the pass@1 metric, \textit{i.e.}, we count the failure or success of each of the five completions for a test case (rather than taking the best out of the five completions).

For the baseline, we compare the publicly available DiffuLLaMA \cite{gong2024scaling} model. We find that the results from \cite{gong2024scaling} are actually under-reported. When manually inspecting the outputs, we found that there were often generations that were correct, except that the number of leading spaces was off by one. Our investigation suggests that the LLaMA 7B tokenizer also seems to have an issue with counting prefix spaces, etc. Since this is not a problem with model capacity, we relaxed the evaluation, and manually inserted the ground truth indentation (\textit{i.e.}, correct number of leading spaces) to DiffuLLaMA's outputs. After doing so, the pass@1 rate increased from about $16\%$ to $40\%$.

This experiment can be run on a 16GB GPU.

\section{Additional Results}\label{subsec:more_results}

\subsection{Any-Subset Speculative Decoding: Comparison of Finetuned and Off-the-Shelf}

\begin{table}[]
    \centering
    \begin{tabular}{l r r r r}
        \rowcolor{gray!35} \textbf{Sampler} & \textbf{Gen PPL} & \textbf{Entropy} & \textbf{NFEs} & \textbf{Time (s)} \\
        \toprule
        %\rowcolor{gray!20} \multicolumn{5}{l}{\textbf{Off-the-Shelf}} \\
        \textbf{\textit{Sequential}} & $59.08 \pm 5.61$ & $3.119 \pm 0.065$ & $486.0 \pm 0.0$ & $18.04 \pm 0.00$ \\
        \textbf{\textit{Speculative}} & $59.12 \pm 5.26$ & $3.206 \pm 0.064$ & $247.3 \pm 1.8$ & $9.36 \pm 0.07$ \\
        \midrule
        \textit{Difference} & $+0.08\%$ & $+2.78\%$ & $-49.12\%$ & $-48.09\%$\\
        % \midrule
        % \rowcolor{gray!20} \multicolumn{5}{l}{\textbf{Finetuned:} $\mathbf{15\% \rightarrow 99\%}$ \textbf{Mask}} \\
        % \textbf{\textit{Sequential}} & $134.29 \pm 2.37$ & $7.654 \pm 0.007$ & $486.0 \pm 0.0$ & $18.19 \pm 0.00$ \\
        % \textbf{\textit{Speculative}} & $133.78 \pm 2.18$ & $7.647 \pm 0.007$ & $414.1 \pm 0.4$ & $15.87 \pm 0.02$ \\
        % \midrule
        % \textit{Difference} & $-0.38\%$ & $-0.09\%$ & $-14.80\%$ & $-12.76\%$ \\
        % \midrule
        % \rowcolor{gray!20} \multicolumn{5}{l}{\textbf{Finetuned:} $\mathbf{90\% \rightarrow 99\%}$ \textbf{Mask}} \\
        % \textbf{\textit{Sequential}} & $109.06 \pm 1.77$ & $7.648 \pm 0.006$ & $486.0 \pm 0.0$ & $18.03 \pm 0.00$ \\
        % \textbf{\textit{Speculative}} & $111.14 \pm 1.87$ & $7.658 \pm 0.006$ & $433.2 \pm 0.4$ & $16.47 \pm 0.02$ \\
        % \midrule 
        % \textit{Difference} & $+1.91\%$ & $+0.14\%$ & $-10.87\%$ & $-8.65\%$ \\
        \bottomrule
    \end{tabular}
    \caption{\textbf{Comparison of \assd (Algorithm \ref{alg:any_order_speculative}) and Sequential Decoding in Off-the-Shelf Model:} The entries show mean and standard error of generative perplexity (judge: GPT-2 Large), Shannon entropy, number of network function evaluations, and wall clock time. Metrics are calculated over $640$ decoded WikiText sequences of length $512$, where $95\%$ is randomly masked out. We set $k=5$ in speculative decoding.}
    \label{tab:more_speculative_decoding}
\end{table}

See Table \ref{tab:more_speculative_decoding}. This shows extended results from Table \ref{tab:speculative_decoding}, with the off-the-shelf model from Huggingface.%, and models with both hyperparameter settings from Section \ref{subsec:hyperparameters} (\textit{i.e.}, the models used in both Tables \ref{tab:speculative_decoding} and \ref{tab:infill_results}).

%\noindent\textbf{Comparison of Finetuned and Off-the-Shelf Model:}
While the generative perplexity of the off-the-shelf model is much lower than the generative perplexity of our finetuned model, this comes at the cost of very low entropy. Indeed, when we manually inspected the outputs, we found that the off-the-shelf model generated highly repetitive, nonsensical sequences of a few common words (see Appendix \ref{sec:off_the_shelf_outputs}). On the other hand, the finetuned model generated sentences that were, for the most part, coherent both semantically and syntatically (see Appendix \ref{sec:sample_outputs}).

% \noindent\textbf{Comparison of Mask Ratios:}
% As expected, the model trained on $90\% \rightarrow 99\%$ masking rate performed the best, as its training distribution was closest to the $95\%$ masks encountered at test time. That being said, the comparison is not perfect, as due to computational limitations, we were not able to use exactly the same hyperparameters otherwise in the $15\% \rightarrow 99\%$ mask-finetuned model. 

%\noindent\textbf{Speedup:} 
Additionally, the off-the-shelf model gains a larger boost in runtime from speculative decoding. Examining the outputted sequences, it again appears to be due to its highly repetitive output distributions, as these would be easier to speculate with a mean field model. 

\section{Ablation}\label{sec:ablations}

\subsection{Mask Decomposition Ablation}

\begin{figure}
    \centering
    \begin{subfigure}[]{0.49\textwidth}
        \includegraphics[width=\textwidth]{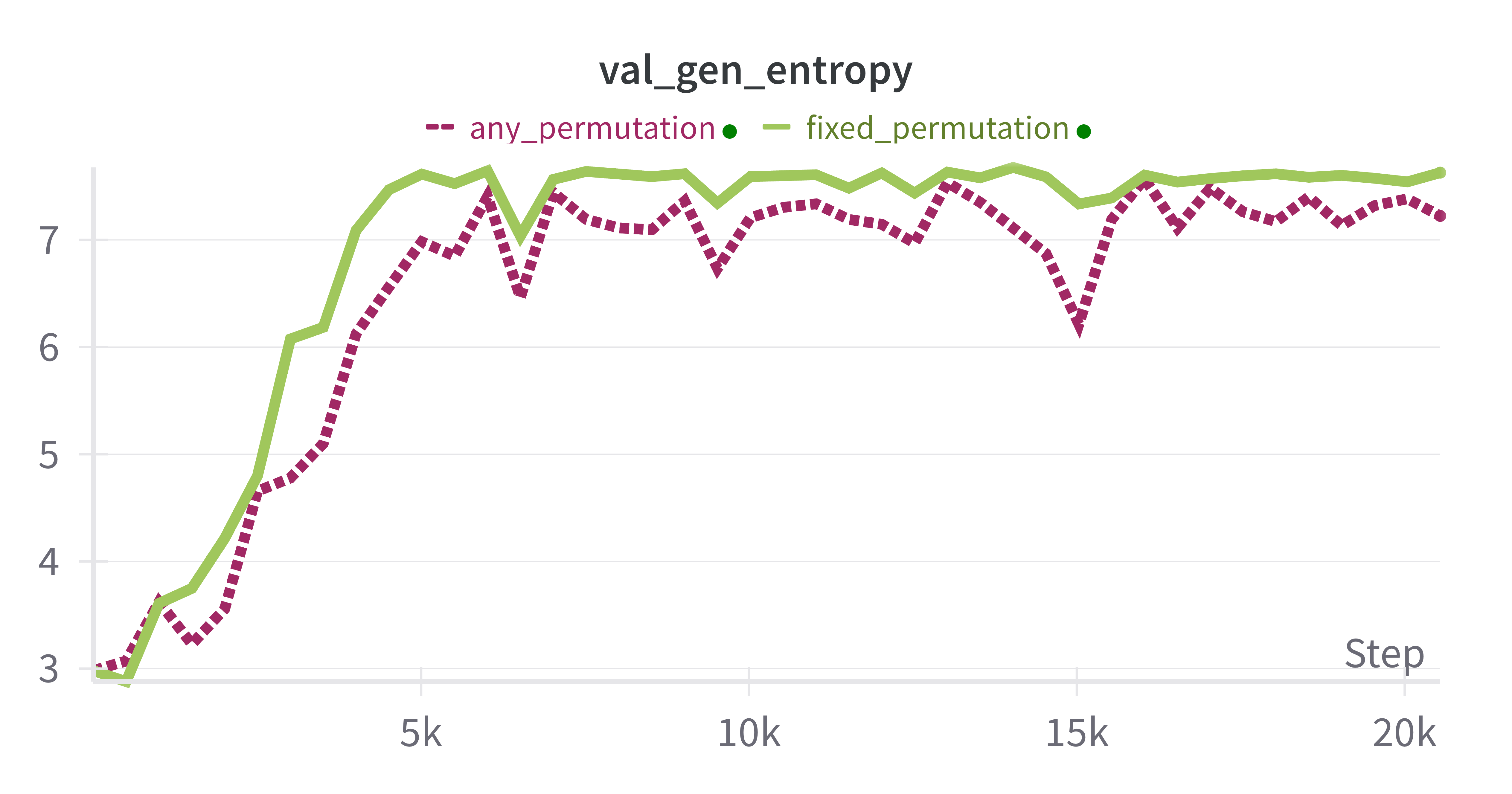}
        \caption{\textbf{Entropy} $(\uparrow)$}
    \end{subfigure}
    \begin{subfigure}[]{0.49\textwidth}
        \includegraphics[width=\textwidth]{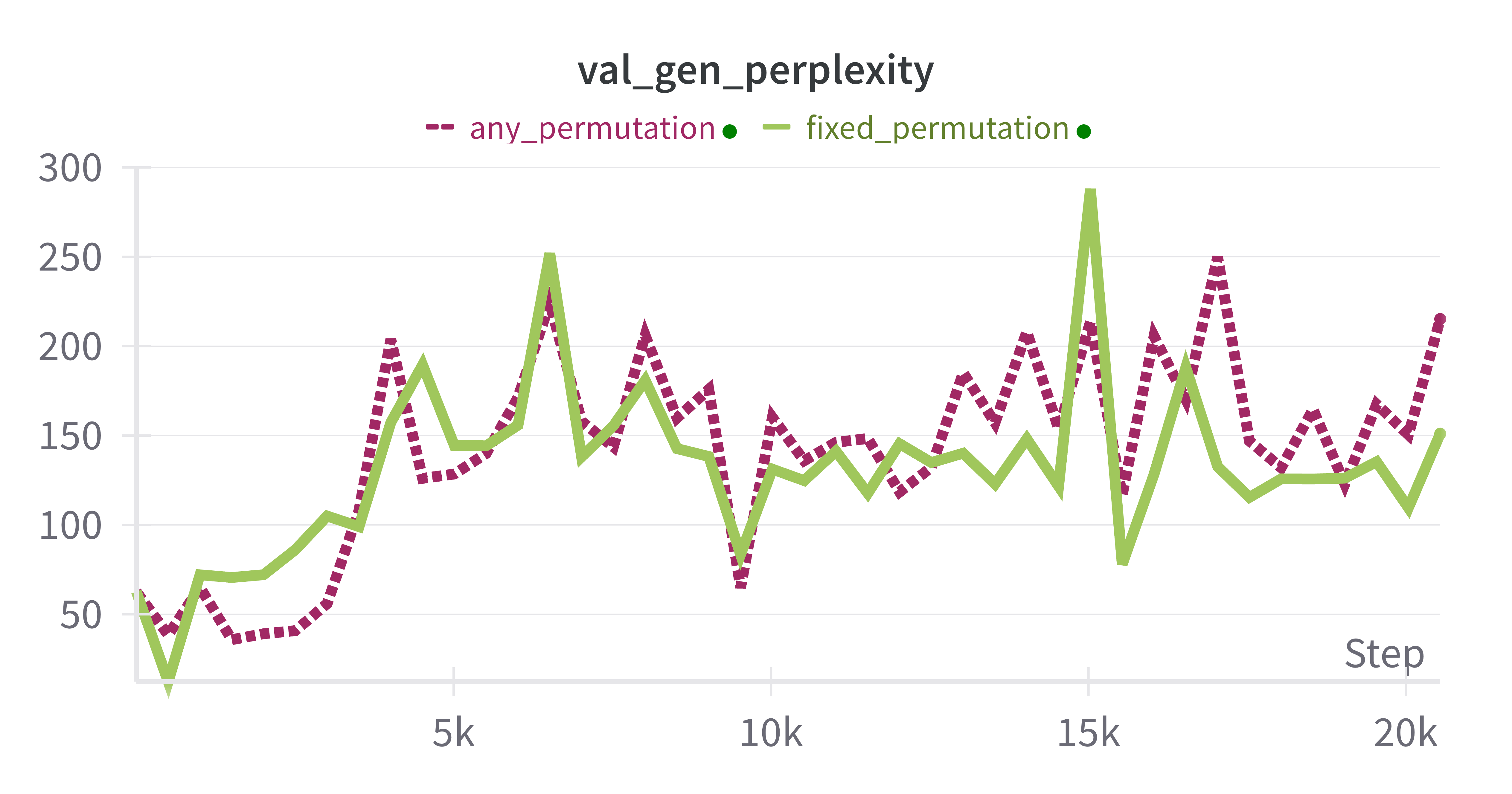}
        \caption{\textbf{Generative Perplexity} $(\downarrow)$}
    \end{subfigure}
    \caption{\textbf{Fixed (Recursive Binary Lattice) Versus Any Permutation Mask Decomposition:} Validation loop metrics on generated sequences from each training strategy. The curves shown are for models trained with an effective batch size of $96$ across four NVIDIA RTX A4000 devices. Each validation iteration has $36$ sequences of length $512$ tokens.}\label{fig:fixed_vs_any_order}
\end{figure}

Figure \ref{fig:fixed_vs_any_order} shows the ablation of mask decomposition protocol described in Section \ref{subsec:efficient_mask}. The entropy (generation diversity) is consistently better when training with the recursive binary mask decomposition protocol, while the generative perplexity is about the same. 

\subsection{Impact of Masking Distribution}

\begin{figure}
    \centering
    \begin{subfigure}[]{0.49\textwidth}
        \includegraphics[width=\textwidth]{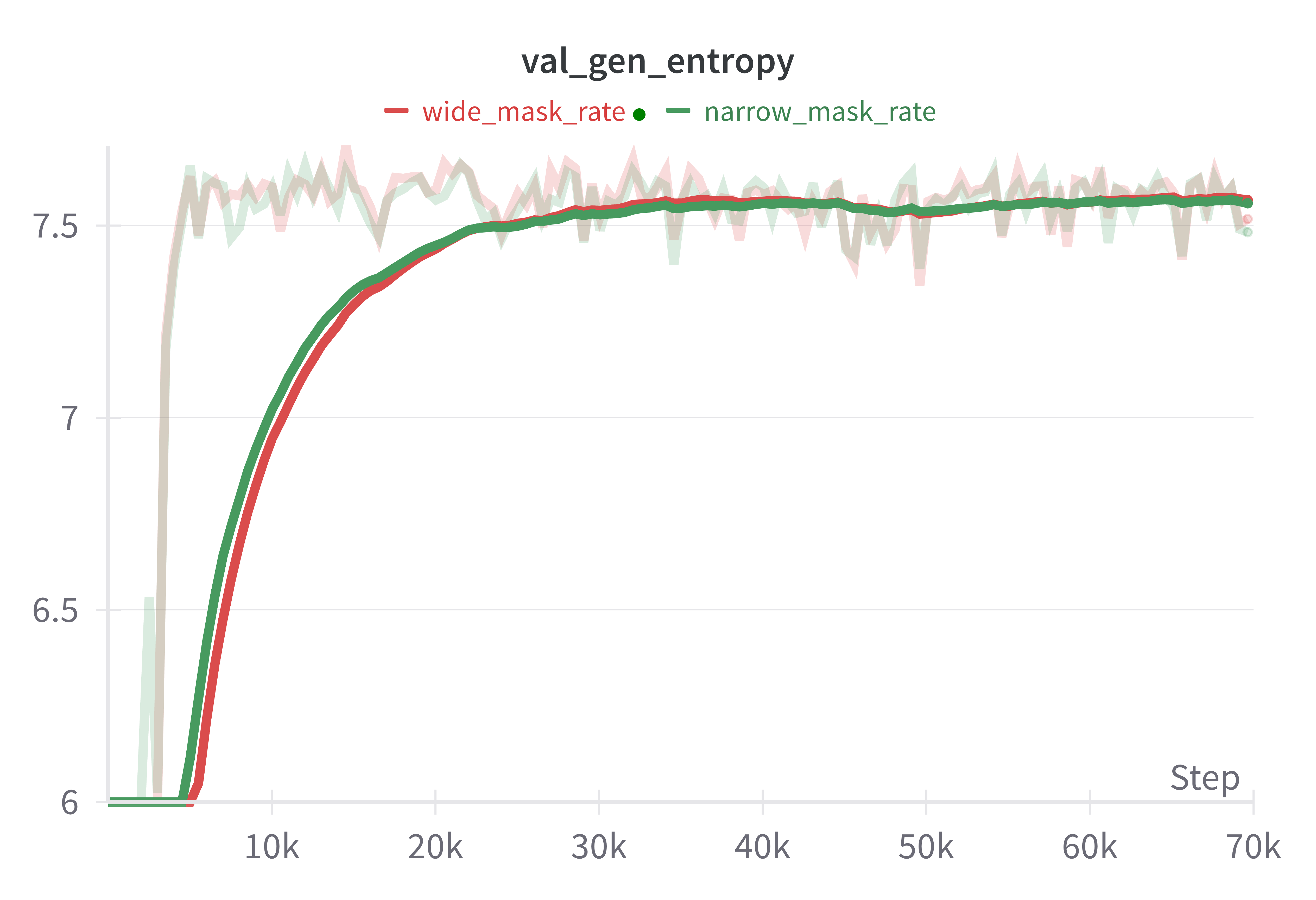}
        \caption{\textbf{Entropy} $(\uparrow)$}
    \end{subfigure}
    \begin{subfigure}[]{0.49\textwidth}
        \includegraphics[width=\textwidth]{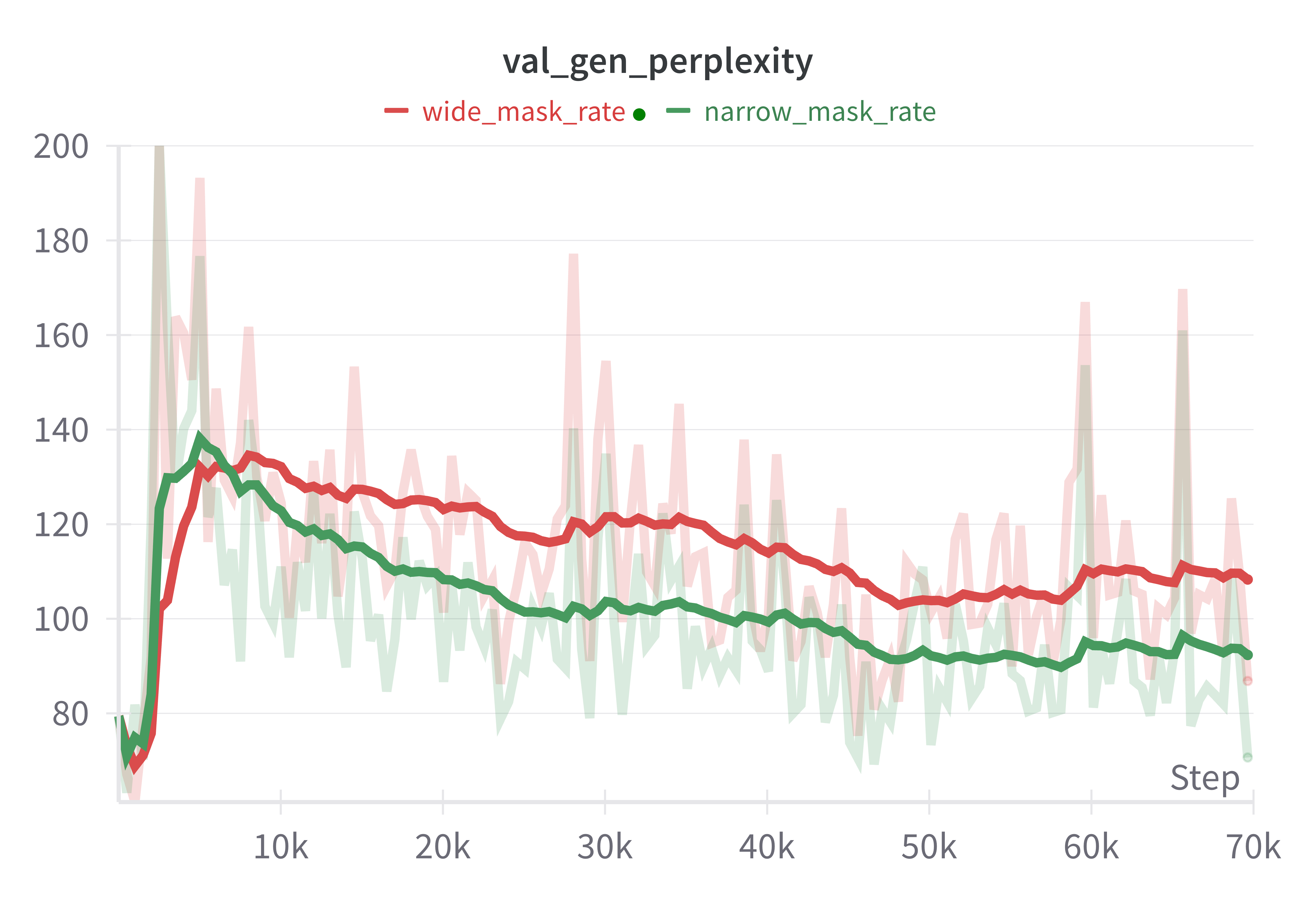}
        \caption{\textbf{Generative Perplexity} $(\downarrow)$}
    \end{subfigure}
    \caption{\textbf{Narrow ($\mathbf{1\%\rightarrow10\%}$) Versus Wide ($\mathbf{1\%\rightarrow85\%}$) Prompting Rates:} Validation loop metrics on generated sequences from each training strategy, as it relates to the distribution of prompt lengths in the train set. The curves shown are for models trained with an effective batch size of $320$ across five NVIDIA RTX 6000 Ada devices. Each validation iteration has $64$ sequences of length $512$ tokens from OpenWebText, where the task is to infill $95\%$ of the masked sequence given a $5\%$ prompt.}\label{fig:narrow_vs_wide}
\end{figure}

Figure \ref{fig:narrow_vs_wide} shows the effect of the realization of the distribution of prompt (\textit{i.e.}, masking) length from Section \ref{subsec:mask_distribution}. Since the validation task is to infill $95\%$ of a masked sequence, it is expected that training the model exclusively on shorter prompt lengths would be better than training the model on a mixture of long and short prompt lengths (which would dilute its capacity). Indeed, we see that, with respect to generative perplexity, the model trained with $m\sim\mathcal{U}[0.01, 0.10]$ outperforms the model trained with $m\sim\mathcal{U}[0.01, 0.85]$, where $m$ is the percentage of the source text that is given as prompt (\textit{i.e.}, \textit{un}masked). On entropy, the short prompt training strategy performs marginally better at low training steps, but the gap closes as training continues.

% \textcolor{red}{gen ppl of (1) narrow mask rate; (2) loss-weighted; (3) wide mask rate (pending)}

\section{Additional Related Works}\label{sec:more_related_works}

\subsection{Any-Order Autoregressive Models}

There are some works from the pre-transformer era \cite{vaswani2017attention} that deal with the problem of arbitrary density estimation, namely NADE \cite{uria2014deep} and MADE \cite{germain2015made}.

\subsection{Discrete Diffusion Models}

There is also work on improving sampling. EDLM \cite{xu2024energy} learns a partition function to allow parallel sampling, while ensuring that the generated tokens adhere to a desired joint distribution. DDPD \cite{liu2024think} learns a planner network to optimize the step size taken with uniform diffusion models. JYS \cite{park2024jump} also works on optimizing the step schedule, although their method involves an optimization problem over (part of) the training dataset. \cite{zhao2024informed} work on predictor-corrector methods and introduce the $k-$Gillespie sampling algorithm. SDTT \cite{deschenaux2024beyond} distilled discrete diffusion models to achieve comparable generative performance with fewer sampling steps. 
However, \cite{zheng2024masked}'s work revealed a numerical precision error with discrete diffusion sampling that leads to unintentional low-temperature sampling. 

%We note that in discrete diffusion models there is an expensive inference cost per model call, as the architectures perform full attention across all tokens, even the masked ones \cite{lou2023discrete, sahoo2024simple}. Crucially, due to the lack of causal attention masking, techniques like KV-caching therefore cannot be applied to speed up discrete diffusion models.

\subsection{Speculative Decoding}

Simple draft models include context-based and model-based n-grams \cite{stewart2024n}. 
Lookahead Decoding also relies upon parallel generation of n-grams, and does not need to train an additional draft model \cite{fu2024break}.
Hydra \cite{ankner2024hydra} and Medusa \cite{cai2024medusa} augment the target language model with lightweight heads to quickly predict additional tokens. The advantage of this approach, like ours, is that at least two tokens are guaranteed to be accepted on each loop iteration: the first token comes from the base model; if the next token is rejected, it is resampled from a combination of the proposal and oracle distributions.
There are also works \cite{zhang2023draft} that, similar to ours, use the same model for drafting and verification; one such work is LayerSkip \cite{elhoushi2024layerskip}. This is known as self-speculative decoding. A difference is that these methods skip model layers, while we skip inputs. 

The only work that we know of that uses a speculative decoding-like algorithm for non-left-to-right models is $\sigma$-GPT \cite{pannatier2024sigma}. However, their algorithm actually violates Theorem \ref{lem:num_evals} and Theorem \ref{thm:correct_dist}, meaning that it is not theoretically guaranteed to produce tokens from the correct target distribution, and can slow down sampling. See Appendix \ref{sec:comparison_to_sigma_gpt} for details.

\section{Comparison to $\sigma$-GPT}\label{sec:comparison_to_sigma_gpt}

\begin{algorithm}
\caption{Token-based rejection sampling, reprinted from $\sigma$-GPT \cite{pannatier2024sigma}}\label{alg:sigma_gpt_rejection}
\SetAlgoLined
\LinesNumbered
\KwIn{$T$: minimum target length, $y$: any-order autoregressive model, $N_o$: number of orderings to sample, $\mathbb{X}$: prompt of length $t_0$}

$t \leftarrow t_0$ \\
\While{t < T}{
    In parallel, compute distribution conditioned on prompt $p(x_i | \mathbb{X}), \forall i \in t, \ldots, T$ \\
    In parallel, sample at every position $\tilde{x}_i \sim p(x_i | \mathbb{X}), \forall i \in t, \ldots, T$ \\
    Draw $N_o$ random order $\sigma$ and in parallel, compute all logits $q\left(x_i | \mathbb{X}, \tilde{x}_{\sigma(<i)}\right), \forall i \in t, \ldots, T$ \label{sigma:line:random_order} \\
    In parallel sample $T - t$ variables $u_i \sim \mathcal{U}[0, 1], \forall i \in t, \ldots, T$ from  a uniform distribution. \\
    In parallel, compute the acceptance decision $a_i = u_i < \text{min}\left(1, \frac{q\left(\tilde{x}_i | \mathbb{X}, \tilde{x}_{\sigma(<i)}\right)}{p(\tilde{x}_i | \mathbb{X})}\right)$ for every order. \label{sigma:line:accept_decision} \\
    Select the order that accepts the most tokens before seeing a first rejection. \label{sigma:line:select_best_order} \\
    Keep that order and add the $a$ accepted tokens before the first rejection to the prompt. \\
    Set $t = t + a$
}
\end{algorithm}

$\sigma$-GPT is a work that superficially seems to have a similar sampling algorithm as ours for any-order autoregressive models \cite{pannatier2024sigma}. However, as we will see, their algorithm (Algorithm \ref{alg:sigma_gpt_rejection}) actually has subtle yet critical mistakes that lead to violations of Theorem \ref{lem:num_evals} and Theorem \ref{thm:correct_dist}, removing the theoretical guarantees.

\subsection{Ambiguous Factorization Order and Joint Distribution}

Firstly, their algorithm does \textit{not} provably give the correct joint distribution, because they randomly sample multiple factorization orders (Lines \ref{sigma:line:random_order}, \ref{sigma:line:select_best_order}) at each draft-accept cycle, given a prompt. This would re-introduce the consistency problem in Equation \ref{eqn:inconsistent}, \textit{i.e.}, different orderings give different joint distributions. Thus, it is unclear what the true target distribution they are trying to match actually is, violating Theorem \ref{thm:correct_dist}. 

Another consequence of the multiple factorization orders sampled is that the number of NFEs could exceed the number of tokens eventually accepted, if the multi-order oracle evaluations all reject the speculated tokens. So, Theorem \ref{lem:num_evals} is violated, which means that this scheme could potentially slow down sampling.

\subsection{Lack of Resampling Step}

Furthermore, their algorithm also does not have a resampling step in case of rejection. This decreases the number of tokens each iteration is guaranteed to accept from two to one (see Theorem \ref{lem:num_evals}'s proof). This means that their algorithm is \textit{not} mathematically guaranteed to reduce the number of function evaluations, and can in theory increase it: in the case that only the first conditionally independent token from the draft is accepted, the function evaluation of the oracle did not lead to an extra token being accepted. This violates Theorem \ref{lem:num_evals}.

Furthermore, the lack of resampling potentially once again violates Theorem \ref{thm:correct_dist}'s guarantee that the returned distribution will match the joint. Using the notation and ideas from Appendix \ref{sec:modified_rejection_sampling}, they have the "accept" (first) term in Equation \ref{eqn:total_prob}, but they do not have the "reject" (second) term to balance the probability out. In fact, their Line \ref{sigma:line:accept_decision} cannot even be considered to be proper rejection sampling, because it lacks the proper normalization for the decision threshold \cite{leviathan2023fast}.

\subsection{Any-Order Speculative Decoding Addresses Problems}

In contrast, by enforcing \cite{shih2022training}'s recursive binary lattice mask decomposition, we ensure that given a prompt, there is \textit{only one} correct path to calculating the joint conditional probability of the missing tokens. This makes it clear in our Algorithm \ref{alg:any_order_speculative} what the target distribution actually is, and our output provably matches that distribution (Theorem \ref{thm:correct_dist}). Furthermore, we have resampling in Line \ref{line:resample}. Combined with the single-path evaluation, this mathematically guarantees that we never increase the number of function evaluations above the number of masked tokens (Theorem \ref{lem:num_evals}).

\section{Off-the-Shelf Model Outputs}\label{sec:off_the_shelf_outputs}

%\subsection{Ground Truth Text}

Text comes from the WikiText dataset \cite{wikitext_merity2016pointer}. This is from the default XLNet model provided on Huggingface (\textit{not} our finetuned), which was only trained to predict around $20\%$ of tokens.

\begin{tcolorbox}[title={Original Text}, colback=blue!10]
<sep><cls></s> = Robert Boulter =<sep><cls></s><sep><cls></s> Robert Boulter is an English film , television and theatre actor . He had a guest @-@ starring role on the television series The Bill in 2000 . This was followed by a starring role in the play Herons written by Simon Stephens , which was performed in 2001 at the Royal Court Theatre . He had a guest role in the television series Judge John Deed in 2002 . In 2004 Boulter landed a role as " Craig " in the episode " Teddy 's Story " of the television series The Long Firm ; he starred alongside actors Mark Strong and Derek Jacobi . He was cast in the 2005 theatre productions of the Philip Ridley play Mercury Fur , which was performed at the Drum Theatre in Plymouth and the Menier Chocolate Factory in London . He was directed by John Tiffany and starred alongside Ben Whishaw , Shane Zaza , Harry Kent , Fraser Ayres , Sophie Stanton and Dominic Hall .<sep><cls></s> In 2006 , Boulter starred alongside Whishaw in the play Citizenship written by Mark Ravenhill . He appeared on a 2006 episode of the television series , Doctors , followed by a role in the 2007 theatre production of How to Curse directed by Josie Rourke . How to Curse was performed at Bush Theatre in the London Borough of Hammersmith and Fulham . Boulter starred in two films in 2008 , Daylight Robbery by filmmaker Paris Leonti , and Donkey Punch directed by Olly Blackburn . In May 2008 , Boulter made a guest appearance on a two @-@ part episode arc of the television series Waking the Dead , followed by an appearance on the television series Survivors in November 2008 . He had a recurring role in ten episodes of the television series Casualty in 2010 , as " Kieron Fletcher " . Boulter starred in the 2011 film Mercenaries directed by Paris Leonti .<sep><cls></s><sep><cls></s> = = Career = =<sep><cls></s><sep><cls></s><sep><cls></s> = = = 2000 – 2005 = = =<sep><cls></s><sep><cls></s> In 2000 Boulter had a guest @-@ starring role on the television series The Bill ; he portrayed " Scott Parry " in the episode , " In Safe Hands " . Boulter starred as " Scott
\end{tcolorbox}

%\subsection{Prompt}\label{subsec:wikitext_prompt}

\begin{tcolorbox}[title={Prompt}, colback=yellow!10]
\_ \_ \_ \_ \_ \_ \_ \_ \_ \_ \_ \_ \_ \_ \_ \_ \_ \_ \_ \_ \_ \_ \_ \_ \_ \_ \_ \_ \_ \_ \_ \_ \_ \_ \_ \_ \_ \_ \_ \_  the television\_ \_ \_ \_ \_ \_ \_ \_ \_ \_ \_ \_ \_ \_  in\_ \_ \_ \_ \_ \_ \_ \_ \_ \_ ,\_ \_ \_ \_ \_ \_ \_ \_  Court\_ \_ \_ \_ \_ \_ \_ \_ \_ \_ \_ \_ \_ \_ \_ \_ \_ \_ \_ \_ \_ \_ \_ \_ \_ \_ \_ \_ \_ \_ \_ \_ \_ \_ \_ \_ \_ "\_ \_ \_ \_ \_ \_ \_ \_ \_ \_ \_ \_ \_ \_ \_ \_ \_ \_ \_ \_ \_ \_  and\_ \_ \_ \_ \_  He\_ \_ \_ \_  2005\_ \_ \_ \_ \_ \_ \_ \_ \_ \_ \_ \_ \_ \_ \_ \_ \_ \_ \_  Plymouth\_ \_ \_ \_ \_ \_ \_ \_ \_ \_ \_ \_ \_ \_ \_ \_ \_ \_ \_ \_ \_ \_ \_ \_ \_ \_ \_ \_ \_ \_ \_ \_ \_ \_ \_ \_ \_ \_ \_ \_  Stanton and\_ \_ \_ \_ \_ \_ \_ \_ \_ \_ \_ \_ \_  starred\_ \_ \_ \_ \_  play Citizens\_ \_ \_ \_ \_ \_ \_ \_ \_ \_ \_  a\_ \_ \_ \_ \_ \_  \_ \_ \_ \_ \_ \_ \_ \_ \_ \_ \_ \_ \_ \_ \_ \_ \_ \_ \_ \_  Josie\_ \_ \_ \_ \_ \_ \_ \_ \_ \_ \_ \_ \_ \_ \_ \_ \_ \_ \_ \_ \_ \_ \_ \_ .\_ \_ \_ \_ \_ \_ \_ \_ \_ \_ \_ \_ \_ \_ \_ \_ \_ \_ \_ \_ \_ \_ \_ key\_ \_ \_ \_ \_ \_ \_ \_ \_ \_ \_ \_ \_ \_ \_ \_ \_ \_ \_ \_ \_ \_ \_ \_ \_ \_ \_ \_ \_ \_ \_ \_ \_ \_ \_ \_ \_ \_ \_ \_ \_ \_ \_ \_ \_ \_ \_ \_ \_ \_ \_ \_ \_ \_ \_ \_ \_ \_ \_ \_ \_ \_ \_ \_ \_ \_ \_ \_ \_ \_ \_ \_ \_ \_ \_ \_ \_ \_ \_ \_ \_ \_ \_ \_ \_ \_ \_ \_ \_ \_ \_ \_  Merc\_ \_ \_ \_ \_ \_ \_ \_ \_ \_ \_ \_ \_ \_ \_ \_ \_ \_  =\_ \_ \_ \_ \_ \_ \_ \_ \_ \_ \_  =\_ \_ \_ \_ \_ \_ \_ \_ \_ \_ \_ \_ \_ \_ \_ \_ \_ \_ \_ \_ \_ \_ \_ \_ \_ \_ \_ \_  the\_ \_ \_ \_ \_ \_ \_  portrayed\_ \_ \_ \_ \_ \_ \_ \_ \_ \_ \_ \_ \_ \_ \_ \_ \_ \_ \_ \_ \_ \_ \_ \_  starred\_ \_ \_ \_ \_ 
\end{tcolorbox}

%\subsection{Off-the-Shelf Model}\label{subsec:off_the_shelf_outputs}

\begin{tcolorbox}[title={Off-the-Shelf, Sequential Decoding}, colback=red!20]
300 300 300 300<eop> Ethnic or Ethnic Ethnically Ly Americans or Others or Others Beyond Ly Ly Ly Ly Ly Ly Ly Ly Ly Ly Ly Ly Ly Ly Ly Ly Ly Ly Ly Ly Ly Ly Ly the television Ly Ly Ly Ly Ly Ly Ly Ly Ly Ly Ly Ly Ly Ly in Ly Ly Ly Ly Ly Ly Ly North ), Ly Ly Ly Ly Ly Ly Ly Ly Court Ly Ly Ly Ly Ly Ly Ly Ly Ly Ly Ly Ly Ly Ly Ly Ly Ly Ly Ly Ly Ly Ly Ly Ly Ly Ly Ly Ly Ly Ly Ly Ly Ly Ly Ly Ly Ly" Ly Ly Ly Ly Ly Ly Ly Ly Ly Ly Ly Ly Ly Ly Ly Ly Ly Ly Ly Ly Ly Ly and Ly Ly Ly Ly Ly He Ly Ly Ly Ly 2005 Ly Ly Ly Ly Ly Ly Ly Ly Ly Ly Ly Ly Ly Ly Ly Ly Ly Ly Ly Plymouth Ly Ly Ly Ly Ly Ly Ly Ly Ly Ly Ly Ly Ly Ly Ly Ly Ly Ly Ly Ly Ly Ly Ly Ly Ly Ly Ly Ly Ly Ly Ly Ly Ly Ly Ly Ly Ly Ly Ly Ly Stanton and Ly Ly Ly Ly Ly Ly Ly Ly Ly Ly Ly Ly Ly starred Ly Ly Ly Ly Ly play Citizens Ly Ly Ly Ly Ly Ly Ly Ly Ly Ly Ly a Ly Ly Ly Ly Ly Ly  Ly Ly Ly Ly Ly Ly Ly Ly Ly Ly Ly Ly Ly Ly Ly Ly Ly Ly Ly Ly Josie Ly Ly Ly Ly Ly Ly Ly Ly Ly Ly Ly Ly Ly Ly Ly Ly Ly Ly Ly Ly Ly Ly Ly Ly. Ly Ly Ly Ly Ly Ly Ly Ly Ly Ly Ly Ly Ly Ly Ly Ly Ly Ly Ly Ly Ly Ly Lykey Ly Ly Ly Ly Ly Ly Ly Ly Ly Ly Ly Ly Ly Ly Ly Ly Ly Ly Ly Ly Ly Ly Ly Ly Ly Ly Ly Ly Ly Ly Ly Ly Ly Ly Ly Ly Ly Ly Ly Ly Ly Ly Ly Ly Ly Ly Ly Ly Ly Ly Ly Ly Ly Ly Ly Ly Ly Ly Ly Ly Ly Ly Ly Ly Ly Ly Ly Ly Ly Ly Ly Ly Ly Ly Ly Ly Ly Ly Ly Ly Ly Ly Ly Ly Ly Ly Ly Ly Ly Ly Ly Ly Merc Ly Ly Ly Ly Ly Ly Ly Ly Ly Ly Ly Ly Ly Ly Ly Ly Ly Ly = Ly Ly Ly Ly Ly Ly Ly Ly Ly Ly Ly = Ly Ly Ly Ly Ly Ly Ly Ly Ly Ly Ly Ly Ly Ly Ly Ly Ly Ly Ly Ly Ly Ly Ly Ly Ly Ly Ly Ly the Ly Ly Ly Ly Ly Ly Ly portrayed Ly Ly Ly Ly Ly Ly Ly Ly Ly Ly Ly Ly Ly Ly Ly Ly Ly Ly Ly Ly Ly; Ly Ly starred Ly Ly Ly Ly Ly
\end{tcolorbox}

\begin{tcolorbox}[title={Off-the-Shelf, Speculative Decoding (Algorithm \ref{alg:any_order_speculative})}, colback=red!10]
series and led starred starred starred starred starred starred starred starred starred starred starred starred starred starred starred starred starred starred starred starred starred starred starred starred starred starred starred starred starred starred starred starred starred starred starred and in the television major starred starred starred starred starred starred starred starred starred starred when the on in became being the title characters into scene as filming the, and too happen which came. The Screen Court Movie""""1996)( were cast they are booked cast to be cast moondicly on set as the movie upon movies too news faint as of taking pres "". They cast to appear on set as of other celebrities the her shee sheeee heee and and is casting as the Hee Hee she 2005 """"Uh Uh Uh Uh Uh Uh Uh, Plymouth listed on set the first here Hee hee she is casting as they are casting to appear as the scheduled here Hee hee he is cast as previously its scheduled her to appear as Stanton and into the Las Vegas during draft her seen her can be cast or starred as the backup for the play Citizens Las Vegas are casting ankle to appear assed or a the next scheduled Las Vegas Las (""""" shortly before after the next scheduled Las Vegas Las Vegas Las "" Josie Las Vegas Las ""jinohlah,x Vegas Las Vegas Las Vegas Las Vegas Las Vegas Las Vegas "." Las Vegas Las Vegas Las Vegas Las Vegas Las Vegas Las Vegas Las Vegas Las Vegas Las Vegas Las Vegas Las Vegaskey Las Vegas Las Vegas Las Vegas Las Vegas Las Vegas Las Vegas Las Vegas Las Vegas Las Vegas Las Vegas Las Vegas Las Vegas Las Vegas Las Vegas Las Vegas Las Vegas Las Vegas Las Vegas Las Vegas Las Vegas Las Vegas Las Vegas Las Vegas Las Vegas Las Vegas Las Vegas Las Vegas Las Vegas Las Vegas Las Vegas Las Vegas Las Vegas Las Vegas Las Vegas Las Vegas Las Vegas Las Vegas Las Vegas Las Vegas Las Vegas Las Vegas Las Vegas Las Vegas Las Las Vegas Las Vegas Las Merc Las Vegas Las Vegas Las Vegas Las Vegas Las Vegas Las Vegas Las Vegas Las Vegas Las Vegas = Las Vegas Las Vegas Las Vegas Las Vegas Las Vegas Las = Las Vegas Las Vegas Las Vegas Las Vegas Las Vegas Las Vegas Las Vegas Las Vegas Las Vegas Las Vegas Las Vegas Las Vegas Las Vegas Las Vegas the Las Vegas Las Vegas Las Vegas Las portrayed Las Vegas Las Vegas Las Vegas Las Vegas Las Vegas Las Vegas Las Vegas Las Vegas Las Vegas Las Vegas Las Vegas Las Vegas starred Las Vegas Las Vegas Las
\end{tcolorbox}

\section{Finetuned Model Outputs}\label{sec:sample_outputs}

Text comes from the WikiText dataset \cite{wikitext_merity2016pointer}. Samples correspond to Table \ref{tab:speculative_decoding}, with our model trained from $90\%\rightarrow99\%$ masking rate, as described in Appendix \ref{subsec:hyperparameters}.

\begin{tcolorbox}[title={Original Text}, colback=blue!10]
<cls></s> AFI 's 10 Top 10 - \# 6 Sports Film<sep><cls></s><sep><cls></s> = = Legacy = =<sep><cls></s><sep><cls></s> In the decades since its release , The Hustler has cemented its reputation as a classic . Roger Ebert , echoing earlier praise for the performances , direction , and cinematography and adding laurels for editor Dede Allen , cites the film as " one of those films where scenes have such psychic weight that they grow in our memories . " He further cites Fast Eddie Felson as one of " only a handful of movie characters so real that the audience refers to them as touchstones . " TV Guide calls the film a " dark stunner " offering " a grim world whose only bright spot is the top of the pool table , yet [ with ] characters [ who ] maintain a shabby nobility and grace . " The four leads are again lavishly praised for their performances and the film is summed up as " not to be missed . "<sep><cls></s> Paul Newman reprised his role as Fast Eddie Felson in the 1986 film The Color of Money , for which he won the Academy Award for Best Actor in a Leading Role . A number of observers and critics have suggested that this Oscar was in belated recognition for his performance in The Hustler . In 1997 , the Library of Congress selected The Hustler for preservation in the United States National Film Registry as " culturally , historically , or aesthetically significant . " Carroll and Rossen 's screenplay was selected by the Writers Guild of America in 2006 as the 96th best motion picture screenplay of all time . In June 2008 , AFI released its " Ten top Ten " — the best ten films in ten " classic " American film genres — after polling over 1 @,@ 500 people from the creative community . The Hustler was acknowledged as the sixth best film in the sports genre .<sep><cls></s> The Hustler is credited with sparking a resurgence in the popularity of pool in the United States , which had been on the decline for decades . The film also brought recognition to Willie Mosconi , who , despite having won multiple world championships , was virtually unknown to the general public . Perhaps the greatest beneficiary of the film 's popularity was a
\end{tcolorbox}

\begin{tcolorbox}[title={Prompt}, colback=yellow!10]
\_ \_ \_ \_ \_ \_ \_ \_ \_ \_ \_ \_ \_ \_ \_ \_ \_ \_ \_ \_ <sep>\_ \_ \_ \_ \_ \_ \_ \_ \_ \_ \_ \_ \_ \_ \_ \_ \_ \_ \_ \_ ,\_ \_ \_ \_ \_ \_ \_ \_ \_ \_ \_ \_ \_ \_ \_ \_ \_ \_ \_ \_ \_ \_ \_ \_ \_ \_  direction\_ \_ \_ \_ \_ \_ \_ \_ \_ \_ \_ \_ \_  editor\_ \_ \_ \_ \_ \_ \_ \_ \_ \_ \_ \_ \_ \_ \_ \_ \_ \_ \_ \_ \_ \_ \_ \_ \_ \_  our\_ \_ \_ \_ \_ \_ \_ \_ \_ \_ \_ \_ \_ \_ \_ \_ \_ \_ \_ \_ \_ \_ \_ \_ \_ \_ \_ \_ \_ \_ \_ \_ \_ \_ \_ \_ \_ \_  \_ \_ \_ \_ \_ \_ \_ \_ \_ \_ \_ \_ \_ \_ \_ \_ \_ \_ \_ \_ \_ \_ \_ \_ \_ \_ \_ \_ \_ \_  the\_ \_ \_ \_ \_ \_ \_ \_ ]\_ \_ \_ \_ \_ \_ \_ \_ \_ \_ \_  and\_ \_ \_ \_ \_ \_ \_ \_ \_ \_  lavish\_ \_ \_ \_ \_  and\_  film\_ \_ \_ \_ \_ \_  not\_ \_ \_ \_ \_ \_ \_ \_ \_ \_ \_ \_ \_ \_ \_  role\_ \_ \_ \_ \_ \_ \_ \_ \_ \_ \_ \_ \_ \_ \_ \_ \_ \_ \_ \_ \_ \_ \_ \_ \_ \_ \_ \_ \_  \_ \_ \_ \_ \_ \_ \_ \_ \_ \_ \_ \_ \_ \_ \_ \_ \_  for\_ \_ \_ \_ \_ \_ \_ \_ \_ \_ \_ \_ \_  Library\_ \_ \_ \_ \_ \_ \_ \_ \_ \_ \_ \_ \_ \_ \_ \_ \_ \_ \_  \_ \_ \_ \_ \_ \_ \_ \_ \_ \_ \_ \_ \_ \_ \_ \_ \_ \_ \_ \_ \_ \_ \_ \_ \_ \_ \_ \_ \_ \_ \_ \_ \_ \_ \_ \_ \_ \_ \_ \_ \_ \_ \_ \_ \_ \_ \_ \_ \_ FI\_ \_ \_ \_ \_ \_  Ten\_ \_ \_ \_ \_ \_ \_ \_ \_  ten\_ \_ \_ \_ \_ \_ \_ \_ \_ \_ \_ \_ \_ \_ \_ \_ \_ \_ \_ \_ \_ \_ \_ \_ \_ \_ \_ \_ \_ \_ \_ \_ \_ \_ \_ \_ \_ \_ \_ \_ \_ \_ \_ \_ \_ \_  Hus\_ \_ \_ \_ \_ \_ \_ \_ \_ \_ \_ \_ \_ \_ \_ \_ \_ \_ \_ \_ \_ \_ \_ \_ \_ \_ \_ \_ \_ \_  also\_ \_ \_  Willie\_ \_ \_ \_ \_ \_ \_ \_ \_ \_ \_ \_ \_ \_ \_ \_ \_ \_ \_ \_ \_ \_ \_ \_ \_ \_ \_ \_ \_ \_ \_ \_  '\_ \_ \_ \_ 
\end{tcolorbox}

\begin{tcolorbox}[title={Finetuned, Sequential Decoding}, colback=green!20]
available to all user accounts. 3. Feel free to throw us a few questions and development suggestions.<sep><cls></s> GENEVA, Switzerland – The paper will not draft further stories, and we will, after all, ignore the party of journalistic freedom, which includes the right not to evaluate claims against anyone, except a dissenting direction from the scientific community. I assure you that we will continue the editor-in-chief role as a quarterly magazine where there will be no exemptions for writers of a particular viewpoint. This will lead our competition to host an active website, WMSU, which will host an automated competition about WMSU as well as to develop accredited methods of talking, researching and writing editorial articles in Swedish ("HF") and English (Roban Xismen in German). History [ edit ] Flawed motives for the art of the book have been debated much.[2] Until 2001, in ten of our 10 articles on 2001 and '02, we have used scientific forms, lavish rigor and style, and fake film. Since then, it has not happened to us. ARandom House, in cooperation with Wikipedia, consists in role-playing games and fantasy series literary novels with a unique approach to artistic media, including "Jeressande" by Audiom et les Hamerswein. See also [ edit ] Table of contents of articles for 2001-11 The intensive research effort of the Fabian Society and the Library and Museum of Great Britain on the RMA 1000 series includes more than nine dazzling articles, 50 "classics" and other pieces that garnered a mention from the Swedish press, an astute academic, sergeant d'Hoskin's Service, an item of art on the rise for WMSU in the 2002 World FI World Games.[3] Ten or more previous articles in February contributed to a ten-years archive by the L.B. Law Foundation of knowledge-urgent WB02 publications data. In 2005, Wercd et al. took the first step toward finishing "Moswart Huslas", which was published in the 1976 edition. In 2007, Munsry et al. repeated "White People's Rescue" also published events by Willie Miller and other Al Jazeera fresh and original voices. In a 2003 paper in the sociology of journalism, David Blabas reporteds, "The so called 'white men's
\end{tcolorbox}

\begin{tcolorbox}[title={Finetuned, Speculative Decoding (Algorithm \ref{alg:any_order_speculative})}, colback=green!10]
it learned to leave the visit.” Read the folksy post—in full here...<sep><cls></s> CSI: Drayton, Part I Sir Ken \& Honi Pratt Expedition Expeditions, Ltd Research \& Development Canada U.S. History and Culture 1,200 Games History Remfacile Rerail will provide nine series of digital direction, redesign and multi-distribution services, leading original biographer and editor David Tatnow appears in the final stages of return service. “It’s almost as fast as walking, getting the basics off our shelves,” Chief Creative Officer Tony Paschus said. “We envy Historic Stadium, Ottawa and the story we can tell with access to our digital resource. The behind-the-scenes humblings cost the taxpayers more than \$14 million last year and were the costliest part of CSI’s entire qualification process isn’t the only aspect of our financial problem [yet]. Players will have access to the most advanced titles, and collectible characters of any franchise. But they will be lavished with additional digital content and new film releases, in things that will not cut through Rui’s new, toned down, strong-arm role—for example, when, in those spacesuits and used as rival generals, she can have dramatic conversations with everyone on her team. (Each character continues to be introduced during missions.) UK flags are symbolic of support for this transition from digital play to entertainment. Credit: Beaujorie Library Release date: 2015-01-25 Format: Remastered digital Download framerate: 57 kbps (a range of 1.2 to 54 kbps) Thumbnail up to 1920x1080 (16 in., 5 notches at 40.1°C) Planting Function: BFI 16 00 | Twelve Tens of catchests Producers: tenthreekorea-reports.gl.ca SUBSCRIBE TO Northeastern University’s CSI: Drayton News \& Games Tour: 51 Place (including Canada), Alberta (including Canada) Huston University University College of Art and Design MacAskill University of Waterloo (Note: Gamers 3 and under are welcome to sign up) See also: All About Willie Lewis<sep><cls></s> Posing for AC/DC's album's tour in October, D'Angelo Brakes of British Columbia proceeded with a classic 'DC' mixta
\end{tcolorbox}

%%%%%%%%%%%%%%%%%%%%%%%%%%%%%%%%%%%%%%%%%%%%%%%%%%%%%%%%%%%%

\end{document}